\documentclass[twocolumn]{autart}    

\usepackage{graphicx}          

\usepackage{amsmath,amssymb,amsfonts}
\usepackage{algorithm,algorithmic}

\usepackage[T1]{fontenc}
\usepackage{lmodern}

\usepackage{textcomp}
\usepackage{enumitem}

\usepackage{cite}

\usepackage[normalem]{ulem}
\usepackage{subcaption}
\newtheorem{theorem}{Theorem}
\newtheorem{lemma}{Lemma}
\newtheorem{remark}{Remark}
\newtheorem{assumption}{Assumption}
\newtheorem{proposition}{Proposition}
\newtheorem{example}{Example}
\newtheorem{corollary}{Corollary}

\usepackage{booktabs}
\usepackage{multirow}
\usepackage{xcolor}

\usepackage[
    colorlinks=true,
    citecolor=blue,
    linkcolor=blue,
    urlcolor=blue
]{hyperref}

\newcommand{\LipMethodName}{{LipLT \noindent}}

\newcommand{\blue}[1]{{\color{black} \noindent #1}}
\newcommand{\blueTwo}[1]{{\color{black} \noindent #1}}

\DeclareMathOperator*{\argmin}{arg\,min}

\begin{document}

\begin{frontmatter}

\title{Provable Bounds on the Hessian of Neural Networks: Derivative-Preserving Reachability Analysis}

\thanks[footnoteinfo]{This paper was not presented at any IFAC 
meeting. Corresponding author M.~Fazlyab.}

\author[JHU]{Sina Sharifi}\ead{sshari12@jhu.edu},    %
\author[JHU]{Mahyar Fazlyab}\ead{mahyarfazlyab@jhu.edu},    %

\address[JHU]{Department of  Electrical and Computer Engineering,
            Johns Hopkins University,
            Baltimore, MD 21218, United States}  


\begin{abstract}                          

    \blue{We propose a novel reachability analysis method tailored for neural networks with smooth activations. Our method relies on constructing an end-to-end first-order Taylor expansion of the network and bounding the remainder. To enable this, we develop a scalable algorithm for computing closed-form bounds on both the gradient norm and Hessian norm of the network. 
    A key component is a loop transformation of the activation functions, which leverages their monotonicity to sharpen the bounds. 
    This yields a local abstraction that preserves the network’s value and gradient at the expansion point. 
    For larger input domains, we integrate this local model into a branch-and-bound scheme, recursively refining the overapproximation. We validate our approach through extensive experiments and demonstrate scalability and accuracy improvements over state-of-the-art verification baselines.}
\end{abstract}

\end{frontmatter}
\section{Introduction}\label{sec:introduction}
The increasing use of neural networks in various applications, ranging from computer vision to control and reinforcement learning, has intensified the need to rigorously verify these models after training, especially in safety-critical domains. This verification task can typically be framed as a constraint satisfaction problem of the form
\begin{align}\label{eq:genericOptimization}
J(f(x)) \leq 0 \quad \forall x\in \mathcal{X},
\end{align}
where \( f \) is a neural network, \( \mathcal{X} \) is a bounded set of inputs, and \( J \) is a scalar-valued function that characterizes the constraint or specification we wish to verify. Defining the optimization problem \( J^\star := \sup_{x \in \mathcal{X}} J(f(x)) \), if the optimal value \( J^\star \) satisfies \( J^\star \leq 0 \), then we have a certificate that the constraint holds. Otherwise, there exists a counterexample \( x^\star \) that violates the constraint. 

\begin{figure}[t]
    \centering
     \includegraphics[width=0.4\textwidth]{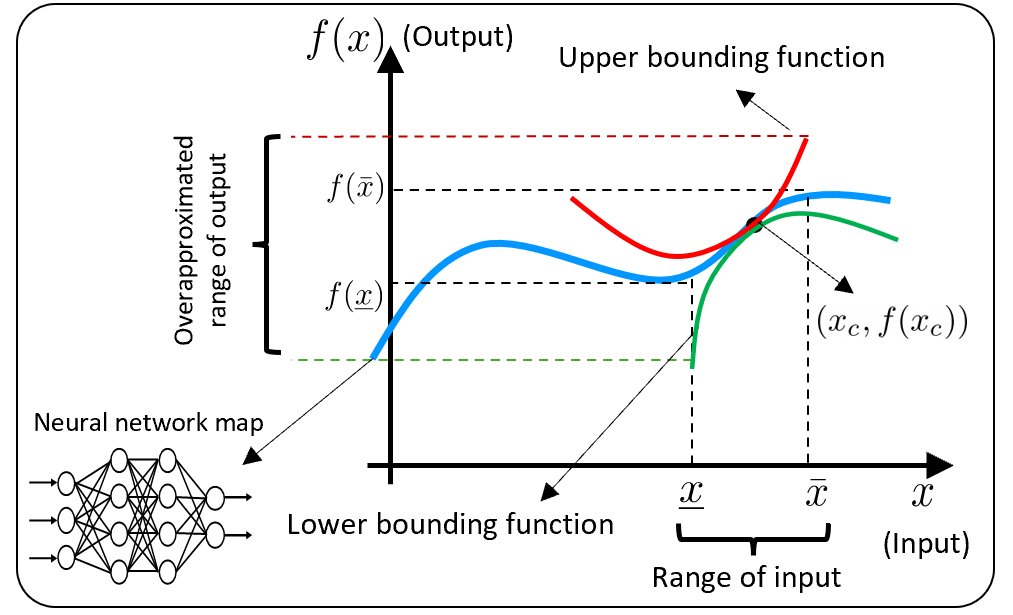}
    \caption{Overview of proposed abstraction.}
    \label{fig:overview}
\end{figure}

The literature on neural network verification has been heavily skewed toward networks with ReLU\footnote{Rectified Linear Unit.} activations. 
Under this setting, \eqref{eq:genericOptimization} can be cast as a Mixed-Integer Linear Program (MILP), which can be solved using Branch-and-Bound (BnB) methods given a sufficient computational budget. 
Due to scalability issues, state-of-the-art verification methods develop custom BnB methods to exploit the problem structure in both the bounding and branching schemes. Since these methods are tailored to ReLU activation functions, they become inherently inapplicable or at least inefficient for general differentiable activation functions such as $\tanh$ or $\mathrm{sigmoid}$. 

In this paper, we develop a verification method specifically tailored for neural networks with twice continuously differentiable activation functions. Our method relies on an end-to-end first-order Taylor abstraction (over-approximation) of the neural network map and bounding the remainder. In contrast to existing linear bound propagation methods, the proposed abstraction preserves the \emph{local gradient information}, allowing it to capture the function's local behavior more accurately (see Fig. \ref{fig:overview}). To bound the remainder, we propose a scalable method to derive analytical bounds on the first derivative (gradient) and second derivative (Hessian) of twice differentiable neural networks. A key aspect of our method is \emph{loop transformation} on the activation functions to exploit their monotonicity in the resulting bounds.




To the best of our knowledge, this paper is the first to develop a \textit{gradient-preserving} sound approximation of neural networks that utilizes second-order derivative information to perform reachability analysis.


\subsection{Contributions}
In summary, focusing on the class of twice continuously differentiable neural networks, we make the following contributions:
\begin{itemize}[leftmargin=*]
    \item We propose a scalable method to derive analytical and differentiable upper bounds on the norm of the first derivative (gradient) of neural networks. To this end, we build on \cite{fazlyab2023certified} to develop \textit{local \LipMethodName}, an algorithm to compute \emph{analytical} upper bounds on the local Lipschitz constant of a neural network. \blue{The method relies solely on matrix-matrix multiplications and matrix norm computations, making it highly scalable to large architectures and amenable to efficient GPU implementation.}
    \item \blue{We extend this framework to develop an algorithm for computing analytical upper bounds on the Hessian of scalar-valued neural networks. This algorithm exploits structural similarities between gradient and Hessian bounds and uses \textit{local \LipMethodName} as a key subroutine, enabling efficient and scalable estimation of second-order sensitivity.} 
    \item \blue{We integrate our derivative-informed bounding approach--based on combined zeroth- and first-order Taylor approximations--into a branch-and-bound (BnB) algorithm over the input domain to enable verification to any prescribed positive tolerance.}
    \item Finally, we present experimental results demonstrating the effectiveness of local LipLT and the success of our BnB algorithm across various reachability benchmarks. The source code can be found at
    \href{https://github.com/o4lc/NNHessianBounds}{GitHub.com/o4lc/NNHessianBounds}.
    
\end{itemize}

This paper is structured as follows. 
In the rest of this section, we cover related work and preliminaries.
Then, in $\S$\ref{sec:ProbStat} we provide the problem formulation and assumptions. 
In $\S$\ref{sec:TaylorApp}, we propose our derivative-informed bounding method for continuously differentiable neural networks. 
In $\S$\ref{sec:LocalLT} and $\S$\ref{sec:hess}, we first develop local \LipMethodName and then utilize it to propose our algorithm to compute bounds on the Hessian of neural networks.
In $\S$\ref{sec:PropMeth} we couple our bounding method with branching to obtain a BnB-based reachability method.
Finally, in $\S$\ref{sec:NumRes} we provide numerical experiments to compare our method with previous works. 

\subsection{Related Work}
Throughout this section, we primarily cover related work on the analysis of neural networks with differentiable activations.

\smallskip \noindent
\textbf{Incomplete Verification Methods:}
Due to the difficulty of computing the optimal value of \eqref{eq:genericOptimization},
incomplete verifiers resort to various convex relaxations \cite{singh2019abstract, anderson2020tightened, bunel2020lagrangian,  hu2020reach, shiformal, wong2018provable,chen2022deepsplit} or bound propagation techniques \cite{gowal2018effectiveness, zhang2018efficient, wang2021beta, entesari2023reachlipbnb} to compute upper bounds on $J^\star$. While some of these methods exploit the derivative information at the neuronal level \cite{fazlyab2020safety, ivanov2021verisig, kochdumper2023open}, none of them preserve the end-to-end derivative. 
%
%
%
Interval Bound Propagation (IBP) \cite{gowal2018effectiveness}, CROWN \cite{zhang2018efficient} and \blue{$\alpha$-CROWN and} Beta-CROWN \cite{xu2020fast, wang2021beta} \blue{obtain linear upper and lower bounds and} are among the bound propagation methods, and \cite{entesari2023reachlipbnb} uses bounds on the Lipschitz constant of $J \circ f$ to find provable bounds for $J^\star$.

\smallskip \noindent
\textbf{Complete Verification Methods:}
Incomplete verification methods tend to produce vacuous bounds when the neural network is deep or the input space is large. 
However, some methods aim to verify the desired property of the network exactly.
To achieve exact certificates, some employ SMT solvers \cite{katz2017reluplex, scheibler2015towards}, while others utilize Mixed-Integer Linear Programs (MILP) \cite{tjeng2017evaluating, dutta2018output, fischetti2018deep} for networks with the ReLU activation functions.
On the other hand, many rely on BnB.
Branching is the action of splitting the space into smaller subspaces, using a sound bounding method to solve the problem on these smaller spaces, thereby achieving overall tight bounds.
The act of splitting can be done either on the input space \cite{entesari2023reachlipbnb, everett2021efficient} or the activation space \cite{tjeng2017evaluating, vincent2021reachable}. 
The main focus of the literature has been branching on ReLU activations, but most recently, \cite{shiformal} has designed a branching heuristic for differentiable activation functions. 

Different heuristics have been proposed to answer two main questions that arise when branching is done on the input space: ``Which subspace is the most promising to branch?'' and once the subspace is chosen, ``How to split that subspace?''
\begin{enumerate}[leftmargin=*]
    \item Splitting the subspace with the maximum volume, the longest edge, the worst node bound \cite{entesari2023reachlipbnb}, based on its contraction rate \cite{ harapanahalli2023contraction} or using Monte Carlo samples \cite{xiang2020reachable, everett2020robustness} have been proposed. 
    \item Splitting along the longest axis \cite{entesari2023reachlipbnb, bunel2018unified}, the use of shadow prices \cite{rubies2019fast}, and choosing the axis that results in the best lower bound \cite{bunel2018unified} are among the proposed heuristics.
\end{enumerate}


\smallskip \noindent
\textbf{Closed-loop Verification:}
In recent years, many formal verification tools have been developed for verification on hybrid systems with neural network controllers. Verisig \cite{ivanov2019verisig} converts the neural network with Sigmoid activation into a hybrid system.
Verisig 2 \cite{ivanov2021verisig}, ReachNN \cite{huang2019reachnn} and ReachNN*\cite{fan2020reachnn} approximate the neural network with Taylor models wherein each neuron is approximated using Taylor series and an associated bound on the error. 
In sharp contrast, we derive an end-to-end Taylor approximation of the network and bound the Hessian, which can in principle reduce the local wrapping effect.
OVERT \cite{sidrane2022overt} abstracts the nonlinearities with piecewise linear functions, and most recently, \cite{kochdumper2023open} approximates each neuron using a polynomial function and propagates the reachable set through the network. 
The authors use polynomial zonotopes as set representations.
Despite the ability of this set representation to capture non-convex geometries, the task of collision checking for safety verification becomes challenging \cite{huang2023difficulty}.

\smallskip \noindent
\textbf{Lipschitz Estimation:}
Many Lipschitz estimation methods focus on ReLU networks \cite{shi2022efficiently, jordan2020exactly, weng2018towards} and only a handful are designed to handle differentiable activations \cite{latorre2020lipschitz}.
In one of the earlier attempts, \cite{szegedy2013intriguing} estimates the global Lipschitz constant as the product of the norms of each layer, which is also known as the Naive Lipschitz bound. 
LipSDP \cite{fazlyab2019efficient} uses monotonicity and Lipschitz continuity of the activation functions to form an SDP that provably returns an accurate numerical upper bound on the Lipschitz constant, but due to the heavy computations of SDP solvers, it is not scalable.  The local version of LipSDP is developed in \cite{hashemi2021certifying}.
Recurjac \cite{zhang2019recurjac} recursively computes the local Lipschitz constant of networks with differentiable activation functions by bounding each element of the Jacobian of the neural network. 

The Lipschitz constant has also been used for robustness verification and training of neural networks \cite{weng2018towards, weng2018evaluating, zhang2019recurjac, fazlyab2019efficient, singla2020second, fazlyab2023certified}. 
\cite{singla2020second} utilizes bounds on the Hessian of neural networks with differentiable activation functions to provide a robustness certificate. Most recently, \LipMethodName was proposed in \cite{fazlyab2023certified} based on loop transformation to develop a scalable method for bounding the Lipschitz constant.

\subsection{Preliminaries and Notation}
    We denote the $n$-dimensional real vector space as $\mathbb{R}^n$. 
    The sequence of natural numbers from $1$ to $n$ is denoted as $[n]$. 
    $\mathrm{1}_{n} \in \mathbb{R}^n$ denotes the column vector where all elements are equal to $1$.
    For a matrix $A \in \mathbb{R}^{n\times m}$, $A^\top_{i} \in \mathbb{R}^{m}$ is the $i$-th row  of $A$ and $A_{i, j}$ is the element on the $i$-th row and the $j$-th column of $A$.
    For vectors  $x, y \in \mathbb{R}^n$, $x \leq y$ denotes element-wise inequality, i.e., $x_i \leq y_i$ for all $ i=1, \dots, n$.
    For a symmetric matrix $A \in \mathbb{S}^n$, we define $A \succeq 0$ and $A \succ 0$ to show that $A$ is positive semidefinite (PSD) and positive definite (PD), respectively. 
    For a vector $x \in \mathbb{R}^n$, $\mathrm{diag}(x)$ denotes a diagonal matrix with $x_i$ on the $i$-th element of the diagonal. $\lambda_{\max} (A)$ and $\lambda_{\min} (A)$ denote the maximum and minimum eigenvalue of $A$, respectively. The range of matrix $A$ is shown as $ \mathcal{R}(A)$, and the pseudo-inverse of $A$ is $A^\dagger$.
    We use $A \odot B$ to denote the Hadamard product of $A, B \in \mathbb{R}^{n\times m}$, where its $j$-th element of the $i$-th row is defined as $(A \odot B)_{i, j} = A_{i, j} B_{i, j}$. 
    For a given hyper-rectangle $\mathcal{X} = \lbrack \ell, u \rbrack := \{x \mid \ell \leq x\leq u\}$, we define $\mathrm{diam}(\mathcal{X}) = \|u - \ell\|_p$.
    An $L$-Lipschitz function $f \colon \mathbb{R}^{n_0} \rightarrow \mathbb{R}$ in $\ell_p$ norm satisfies
    \begin{align}
        |f(x) - f(y)| \leq L \|x - y\|_p \qquad \forall x, y \in \mathbb{R}^{n_0}. \notag
    \end{align}
    An immediate implication of Lipschitz continuity is that, wherever the gradient exists, its norm is upper bounded by the Lipschitz constant
    $
        \|\nabla f(x)\|_{p^*} \leq L \ \forall x \in \mathbb{R}^{n_0}, \notag
    $
    where $p,p^*$ are H\"{o}lder conjugates, i.e.,  $\frac{1}{p}+\frac{1}{p^*}=1$. 
    We will drop the subscript $p$ when the context is clear and show the dual norm of $\|\cdot\|$ by $\|\cdot\|_*$. 
    Furthermore, the following inequality between the $p \geq 1$ and the $q\geq 1$ norms of a vector $z$ holds.
    \begin{align}\label{eq:normIneq}
        n_0^{\min (0, \frac{1}{p} - \frac{1}{q})} \|z\|_q 
        \leq \|z\|_p \leq
        n_0^{\max (0, \frac{1}{p} - \frac{1}{q})} \|z\|_q
    \end{align}
    For a vector $x \in \mathbb{R}^{n_0}$, $\mathrm{sign}(x)$ denotes the element-wise sign function. 
    The operator norm of a matrix $A$ is denoted as
    $
    \|A\|_{p\rightarrow q} = \sup_{\|x\|_p \leq 1} \|Ax\|_q
    $, and $|A|$ denotes the element-wise absolute value of $A$. 
    For simplicity, we write $\|A\|_p := \|A\|_{p\to p}$ for the induced matrix $p$-norm. Furthermore, we have $\|A\|_{p\to p} = \|A^\top\|_{p^*\to p^*}$.


\section{Problem Statement}\label{sec:ProbStat}
Given a continuous function $f \colon \mathbb{R}^{n_x} \rightarrow \mathbb{R}^{n_f}$ and a compact non-empty set $\mathcal{X} \subset \mathbb{R}^{n_x}$, the task of reachability analysis amounts to computing the image of $\mathcal{X}$ under $f$, 
\begin{align}
    \mathcal{Y} = f(\mathcal{X}):=\{f(x) \mid x \in \mathcal{X}\}.  \notag
\end{align}
Since finding the exact reachable set is computationally prohibitive, the goal is to typically find an accurate yet efficient \emph{over-approximation} of the image set. 
In this paper, we assume that $f$ is a neural network with sufficiently smooth activation functions such as $\tanh$, sigmoid, and softplus.

\smallskip \noindent
\textbf{Reachable set representation:}
We focus on polyhedral over-approximations of the reachable set
$f(\mathcal{X})$. For a given direction $c\in\mathbb{R}^{n_f}$,
consider the optimization problem
\begin{align}\label{eq:boundProblem}
    J^*_{\mathcal{X}}
    := \sup_{x\in\mathcal{X}} J(x),
    \qquad J(x):=c^\top f(x).
\end{align}
The optimal value $J^*_{\mathcal{X}}$ is the support function
of $f(\mathcal{X})$ evaluated at $c$.
Evaluating this support function along finitely many
directions yields a polyhedral outer approximation of
the convex hull of $f(\mathcal{X})$.

\smallskip \noindent
\textbf{Input set:}
We initially consider an $\ell_p$-norm ball
$\mathcal{X}=\{x_c+\delta\mid\|\delta\|_p\leq\varepsilon\}$,
where $x_c\in\mathbb{R}^{n_x}$ is the center and
$\varepsilon>0$ is the radius.
We subsequently extend the framework to input sets of the form
$\mathcal{X}=\{x_c+Gz\mid\|z\|_p\leq 1\}$,
where $G\in\mathbb{R}^{n_x\times m}$ is a generator matrix
and $z\in\mathbb{R}^m$.
In particular, $p=2$ and $p=\infty$ correspond to
ellipsoids and zonotopes, respectively.

\smallskip \noindent
\textbf{Neural network model:}
We assume that $f \colon \mathbb{R}^{n_0} \rightarrow \mathbb{R}^{n_L}$ ($n_0 = n_x, n_L = n_f$) is an $L$-layer fully-connected neural network of the following form,
    \begin{align}\label{eq:NNdef}
            z^{(l)}(x) &= W^{(l)} a^{(l-1)}(x) + b^{(l)} \\ \notag
            a^{(l)}(x) &= \sigma^{(l)}(z^{(l)}(x))
        \quad l=1, \cdots, L-1 \\ \notag
        f(x) &= z^{(L)}(x) = W^{(L)} a^{(L-1)}(x) + b^{(L)},
    \end{align}
$a^{(0)}(x) = x$ is the input, $\sigma^{(l)} \colon \mathbb{R}^{n_{l}} \to \mathbb{R}^{n_{l}}$ is the activation layer of the form $\sigma^{(l)}(z) = (\sigma^{(l)}_{1}(z_1),\cdots,\sigma^{(l)}_{n_1}(z_{n_{l}}))$ with $\sigma^{(l)}_i \colon \mathbb{R} \to \mathbb{R}$ being a smooth activation function, $n_l$ is the number of activation units of the $l$-th layer, and $W^{(l)}, b^{(l)}$ are weights and biases of that layer, respectively. 
Whenever the dependencies on $x$ are evident, we omit it for brevity. 
\begin{assumption}[Smooth Activation Functions]\label{asmp:activations}
    Each activation function $\sigma^{(l)}_i \colon \mathbb R \to \mathbb R$ is twice continuously differentiable and slope-restricted for some $0 \leq \alpha^{(l)}_i \leq \beta^{(l)}_i < \infty$, i.e.,
    \begin{align}\label{eq:activSlope}
        \alpha^{(l)}_i \leq \frac{\sigma^{(l)}_i(x) - \sigma^{(l)}_i(y)}{x-y} \leq \beta^{(l)}_i \qquad \forall x, y \in \mathbb{R}.
    \end{align}
    Furthermore, the derivative $\sigma'^{(l)}_i \colon \mathbb R \to \mathbb R$ is slope-restricted for some $ -\infty < \alpha'^{(l)}_i \leq \beta'^{(l)}_i <\infty$, i.e.,
    \begin{align}\label{eq:activCurv}
        \alpha'^{(l)}_i \leq \frac{\sigma'^{(l)}_i(x) - \sigma'^{(l)}_i(y)}{x-y}  \leq \beta'^{(l)}_i \qquad \forall x, y \in \mathbb{R}. 
    \end{align}
\end{assumption}
\medskip
Inequality \eqref{eq:activSlope} implies that the $\sigma^{(l)}_i$ is Lipschitz continuous with constant $\beta^{(l)}_i$, and is (strongly) monotone with constant $\alpha^{(l)}_i$. Together they imply the bounds 
\[
|\sigma'^{(l)}_i(x)| \leq \max (\alpha^{(l)}_i, \beta^{(l)}_i) := g_i^{(l)}
\]
and
$$
|\sigma''^{(l)}_i(x)| \leq \max (|\alpha'^{(l)}_i|, |\beta'^{(l)}_i|) := h_i^{(l)}
$$
for all $x \in \mathbb{R}$. As an example, $\tanh$ satisfies the aforementioned properties with $\alpha = 0$, $\beta =1$, $\alpha'=-\beta'=-\frac{4}{3\sqrt{3}}$. 

Although Assumption~\ref{asmp:activations} is stated globally, the slope and curvature bounds can be localized to the range of pre-activations induced by the input set $\mathcal{X}$.
Such local bounds can be substantially tighter than their global counterparts and, consequently, lead to tighter reachability estimates. Throughout the remainder of the paper, we assume that this localization is performed as a preprocessing step and use
$\alpha^{(l)}$ and $\beta^{(l)}$ to denote the resulting vectors of local slope bounds for layer $l$.

\section{Sound Taylor Approximation of Neural Networks}\label{sec:TaylorApp}
In this section, we develop sound bounds on the optimal value
of \eqref{eq:boundProblem} over the norm ball
$\mathcal{X}=\{x_c+\delta\mid\|\delta\|_p\leq\varepsilon\}$,
using local Taylor approximations of the scalar-valued function
$J:\mathbb{R}^{n_x}\to\mathbb{R}$. Specifically, we seek
lower and upper bounds satisfying
$\underline{J}_{\mathcal{X}}
\leq J_{\mathcal{X}}^* 
\leq \overline{J}_{\mathcal{X}}.$
We first establish zeroth-order bounds based on the Lipschitz
continuity of $J$, and then derive first-order bounds that
exploit local gradient and curvature information.
In $\S$\ref{sec:PropMeth}, we incorporate these bounds
into a branch-and-bound framework to approximate
$J_{\mathcal{X}}^*$ to arbitrary accuracy.

\subsection{Zeroth-order approximation}
Suppose $J$ is Lipschitz continuous over $\mathcal{X}$ with Lipschitz constant $L>0$, i.e., for all  $x, y \in \mathcal{X}$,
\begin{align} \label{eq: zeroth order expansion}
   J(y) \!-\! L\|x\!-\!y\|_p 
   \leq J(x) \leq J(y) \!+\! L\|x\!-\!y\|_p.
\end{align}
Setting $y=x_c$ and taking the supremum over $x$ yields
\begin{equation} \label{eq:zeroth-orderbounds}
\begin{aligned}
    \underline{J}^0_\mathcal{X}(x_c) &= \sup_{x \in \mathcal{X}} \{J(x_c)\! -\! L \|x-x_c\|_p\} = J(x_c) \\
    \overline{J}^0_\mathcal{X}(x_c) &= \sup_{x \in \mathcal{X}} \{J(x_c)\! + \! L \|x-x_c\|_p\} = J(x_c)\! +\! L \varepsilon.
\end{aligned}
\end{equation}
Since we are only assuming Lipschitz continuity, these bounds apply to neural networks with non-smooth activations, including piece-wise linear networks such as ReLU networks.

If $J$ is continuously differentiable, Lipschitz continuity further implies $\|\nabla J(x)\|_{p^*}\leq L \ \  \forall x \in \mathcal{X}$. Thus, the zeroth-order approximation uses only a uniform bound on the gradient norm and does not exploit the value or direction of the gradient at a particular point. We next show that incorporating this local gradient information, together with curvature bounds,
can yield locally tighter approximations.

\subsection{First-order approximation}
Assume that $J$ is twice continuously differentiable. For any $x,y\in\mathcal{X}$, the first-order Taylor expansion of $J$ around $y \in \mathcal{X}$, with the Lagrange form of the remainder is 
\begin{align}
    J(x)\! =\! J(y) \!+\! \nabla J(y)^\top (x\!-\!y) \! + \! 
    \frac{1}{2} (x\!-\!y)\!^\top\!\nabla^2 J(\xi) (x\!-\!y), \notag
\end{align}
where $\xi = (1-\theta) y + \theta x$ and $\theta \in (0,1)$. Since $\mathcal{X}$ is convex, $\xi \in \mathcal{X}$.  Suppose there exist symmetric matrices $N, M \in \mathbb{S}^{n_x}$ such that 
\begin{align} \label{eq: bounds on Hessian}
N \preceq \nabla^2 J(x) \preceq M \quad \forall x \in \mathcal{X}. 
\end{align}
Using these bounds in the preceding Taylor expansion, 
%
%
and taking the supremum over $x$, for any fixed $y \in \mathcal{X}$, yields
\begin{align}
    \underline{J}^1_\mathcal{X}(y)
    \leq \sup_{x \in \mathcal{X}} J(x) 
    \leq \overline{J}^1_\mathcal{X}(y),
\end{align}
where 
\begin{subequations}\label{eq:SecondOrderLowerBound}
\begin{align}
    \underline{J}^1_\mathcal{X}(y) \! &:=\sup_{x \in \mathcal{X}} (J(y) \! + \! \nabla J(y)\!^\top \! (x \! - \! y) \! + \! \tfrac{1}{2} (x \! - \! y) \! ^\top \!  N (x \! - \!y))\label{eq:subeq1} \\
    \overline{J}^1_\mathcal{X}(y)\!  &:=  \sup_{x \in \mathcal{X}} (J(y) \!+\!  \nabla J(y)\!^\top \! (x \! - \! y) \! + \! \tfrac{1}{2} (x \! - \! y) \! ^\top \!  M (x \! - \!y))\label{eq:subeq2}
\end{align}
\end{subequations}
As we show in \S\ref{sec:PropMeth}, for sufficiently small \(\mathcal{X}\), the gradient-informed
bounds \eqref{eq:SecondOrderLowerBound} strictly improve upon the zeroth-order bounds~\eqref{eq:zeroth-orderbounds}. In \S\ref{sec:LocalLT}, we develop a scalable algorithm to compute the Lipschitz constant \(L\),
which is further specialized in \S\ref{sec:hess} to obtain analytical expressions for the
curvature bounds \(M\) and \(N\).

\paragraph*{Comparison to Linear Bound Propagation}
A key distinction of our approach is that it preserves local gradient information through
\(\nabla J(y)\).
By contrast, symbolic bound propagation methods such as CROWN \cite{zhang2018efficient} construct affine upper and lower bounds over an input region \(\mathcal{X}\) (often centered at \(x_c\)) without matching the
gradient at any particular point.
Appendix~\ref{sec:appendix} provides an illustrative example highlighting this difference.

\section{Local Lipschitz approximation for general neural networks} \label{sec:LocalLT}
LipSDP~\cite{fazlyab2019efficient} provides numerical upper
bounds on the Lipschitz constant of neural networks by
encoding activation functions through quadratic constraints
and formulating Lipschitz estimation as an SDP. Although LipSDP can yield tight bounds,
its computational cost limits its scalability to larger
networks. To address this limitation, we build on
\LipMethodName~\cite{fazlyab2023certified}, which provides
scalable analytical upper bounds on the global Lipschitz
constant, potentially at the expense of tightness.
We extend this approach by incorporating input-dependent
slope bounds, yielding a localized variant of
\LipMethodName that exploits the restricted input domain
$\mathcal{X}$. We first present the construction for
two-layer networks and then extend it to general
multi-layer architectures.

\subsection{Two-layer neural networks}
\label{sec:twoLayer}

Consider a two-layer neural network of the form
\[
f(x)=W^{(2)}\,\sigma\!\left(W^{(1)}x\right),
\]
where \(W^{(1)}\in\mathbb{R}^{n_1\times n_x}\), \(W^{(2)}\in\mathbb{R}^{n_f\times n_1}\), and
\(\sigma:\mathbb{R}^{n_1}\to\mathbb{R}^{n_1}\) acts componentwise as
\(
\sigma(z)=\big(\sigma_1(z_1),\ldots,\sigma_{n_1}(z_{n_1})\big).
\)
Assume each nonlinearity is slope-restricted, i.e., \(\sigma_i \in \mathrm{slope}(\alpha_i,\beta_i)\),
and define \(\alpha:=(\alpha_1,\ldots,\alpha_{n_1})\) and \(\beta:=(\beta_1,\ldots,\beta_{n_1})\). For a vector \(d\in\mathbb{R}^{n_1}_{+}\), let \(D:=\mathrm{diag}(d)\) and define
the loop-transformed nonlinearity
\[
\psi(z;d) := \sigma(z) - Dz.
\]
The diagonal transformation \(D\) applies a coordinate-wise loop transformation to each activation.
With this definition, the network admits the decomposition (see Fig.~\ref{fig:LT})
\begin{align}\label{eq:liplt2layer}
f(x)
&= W^{(2)}\Big(\psi(W^{(1)}x;d) + DW^{(1)}x\Big).
\end{align}

Our goal is to derive a parameterized upper bound $L(d)$
on the Lipschitz constant of $f$, which can subsequently
be tightened by optimizing over the loop-transformation
parameter $d$. The following lemma establishes this bound.

\begin{lemma}\label{lemma:singlelaterLT}
Consider the loop-transformed two-layer network \eqref{eq:liplt2layer}. Define
\begin{align} \label{eq: LipLT parameterized bound two layer}
    L(d)\!=\!\|W^{(2)}\|_p\|&\mathrm{diag}(\max(|\beta \!-\! d|, |d \!-\! \alpha|))W^{(1)}\|_p  \notag \\
             &+\|W^{(2)}\mathrm{diag}(d)W^{(1)}\|_p. 
\end{align}
where \(\max(\cdot,\cdot)\) and absolute values are taken elementwise. Then, the Lipschitz constant of \(f\) in \(\ell_p\)-norm is bounded by $L(d)$.
\end{lemma}
See Appendix~\ref{subsec:proofs} for the proof.
\begin{figure}
    \centering
     \includegraphics[width= 0.48\textwidth]{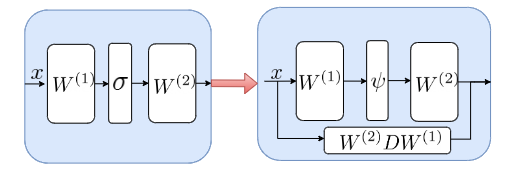}
    \caption{\blue{\small Loop transformation on a two-layer neural network.}}
    \label{fig:LT}
\end{figure}
Setting $d=0$ in \eqref{eq: LipLT parameterized bound two layer}
recovers the naive Lipschitz bound obtained without loop
transformation:
\begin{align}\label{eq:L0_naive}
    L(0)=\|W^{(2)}\|_{p}\,
    \|\mathrm{diag}(\beta)W^{(1)}\|_{p},
\end{align}
where we used $0\leq\alpha\leq\beta$. This bound can
potentially be tightened by optimizing the
loop-transformation parameter:
\begin{align}\label{eq:lipLTOpt}
    d^\star\in\arg\min_{d\in\mathbb{R}_+^{n_1}} L(d).
\end{align}
Although this optimization can be solved numerically,
the following theorem provides a simple analytical choice
of $d$ that is guaranteed to yield a bound no worse than
the naive estimate $L(0)$.

\begin{theorem}[Analytical loop transformation]
\label{thm:beta_half_improves}
Assume \(0\le \alpha \le \beta\) elementwise. Then, for \(d=\beta/2\),
\begin{align}
L(\beta/2) \le L(0).
\end{align}
\end{theorem}
\begin{proof}
Substituting $d=\beta/2$ in \eqref{eq: LipLT parameterized bound two layer}, we can write
\begin{align}
    L(\frac{\beta}{2})\! &= \! \|W^{(2)}\|_p\|\mathrm{diag}(\tfrac{\beta}{2})W^{(1)}\|_p \! + \! \|W^{(2)}\mathrm{diag}(\tfrac{\beta}{2})W^{(1)}\|_p \notag \\
    &\leq \! \|W^{(2)}\|_p \|\mathrm{diag}(\beta) W^{(1)}\|_p = L(0), \notag
\end{align}
where in the first line, we used the fact that $\max(|\beta - \beta/2|, |\beta/2 - \alpha| = \beta/2$ when $\alpha \geq 0$.
\end{proof}

\paragraph*{Loop transformation: bridging naive bounds and LipSDP.}
The naive Lipschitz bound $L(0)$ treats each activation
as a generic Lipschitz nonlinearity, without exploiting
its monotonicity. In contrast, LipSDP~\cite{fazlyab2019efficient} incorporates both
slope restrictions and monotonicity through quadratic
constraints, but at a substantially higher computational
cost. Loop transformation provides an intermediate
approach by decomposing each activation into a linear
component and a residual nonlinearity.
The linear component introduces an additional
feedforward term, while the residual has a potentially
smaller Lipschitz constant.
Bounding these components separately yields an analytical
and scalable estimate that exploits the activation's slope
restrictions. This decomposition forms the basis of
both the global and localized versions of \LipMethodName.

\paragraph*{Comparison with global \LipMethodName~\cite{fazlyab2023certified}}\label{rmk:compareLipLT}
Global \LipMethodName assumes \emph{homogeneous} slope bounds across all activation units, i.e.,
\(\sigma_i \in \mathrm{slope}(\tilde{\alpha},\tilde{\beta})\) for \(i=1,\ldots,n_1\), and restricts the loop
transformation to be shared across coordinates, \(d=\tilde d\,\mathbf{1}_{n_1}\).
Under these assumptions, \cite{fazlyab2023certified} shows that the optimal transformation is
\(
d^\star = \frac{\tilde{\alpha}+\tilde{\beta}}{2}\,\mathbf{1}_{n_1}.
\)
In contrast, the localized setting studied here leads to \emph{heterogeneous} bounds
\(\sigma_i \in \mathrm{slope}(\alpha_i,\beta_i)\), which naturally motivates coordinate-dependent loop
transformations \(d\in\mathbb{R}^{n_1}_{+}\).
In this regime, the direct analogue \(d=(\alpha+\beta)/2\) need not improve the naive estimate \(L(0)\),
since the added feedforward term \(\|W^{(2)}DW^{(1)}\|_p\) can dominate the reduction in the residual
nonlinearity. The following example illustrates this phenomenon.


\begin{example}\normalfont
Consider a two-layer neural network $f : \mathbb{R}^2 \to \mathbb{R}^2$ with $W^{(1)} = \begin{bmatrix} 1 & 2 \\ 1 & 2\end{bmatrix}$ and 
$W^{(2)} = \begin{bmatrix} 1 & 1 \\ 1 & 2\end{bmatrix}$. 
Let $\alpha = \begin{bmatrix} 0.2 \\ 0.6 \end{bmatrix}$ and $\beta = \begin{bmatrix} 0.8 \\ 0.7 \end{bmatrix}$. 
For this network, bounds on the Lipschitz constant using the naive approach $L(0)$, LipLT with $d = \frac{\alpha + \beta}{2}$,  $L(\frac{\alpha + \beta}{2})$, our choice $L(\frac{\beta}{2})$, and the optimal loop transformation $d^\star$ can be computed as
\begin{align} \begin{cases} L(0) = \|W^{(2)}\|_2\|\mathrm{diag}(\beta) W^{(1)}\|_2 = 6.22 \notag \\ L(\frac{\alpha + \beta}{2}) = 0.5(\|W^{(2)}\mathrm{diag}(\beta + \alpha) W^{(1)}\|_2 + \notag \\ 
\qquad \qquad \qquad \quad \|W^{(2)}\|_2\|\mathrm{diag}(\beta - \alpha) W^{(1)}\|_2) \!=\! 6.56\notag \\ 
L(\frac{\beta}{2}) = 0.5(\|W^{(2)}\mathrm{diag}(\beta) W^{(1)}\|_2 \!+ \notag \\ 
\qquad \qquad \qquad \quad \|W^{(2)}\|_2\|\mathrm{diag}(\beta) W^{(1)}\|_2) \!=\! 6.09 \notag\\ 
L(d^\star) = 5.96 \notag 
\end{cases} \end{align}
In particular, the homogeneous ``balancing'' choice \(d=(\alpha+\beta)/2\), which is optimal in the
global (homogeneous-slope) setting of~\cite{fazlyab2023certified}, can be worse than the naive bound
\(L(0)\) under heterogeneous slope bounds.
Here \(d^\star\) was computed in \texttt{CVXPY}~\cite{cvxpy} by solving \eqref{eq:lipLTOpt} under the constraints \(0\le d_i \le (\alpha_i+\beta_i)/2\), which make the objective convex.
\end{example}

\subsection{Multi-layer neural networks}
We now extend the loop-transformation construction to the \(L\)-layer network \eqref{eq:NNdef}. Since the bias terms cancel when taking differences between network outputs at two inputs, we omit them for notational simplicity. We apply a layer-dependent loop transformation \(d^{(l)}\) to each hidden layer \(l=1,\ldots,L-1\). 
For each layer \(l\), let \(D^{(l)}:=\mathrm{diag}(d^{(l)})\) with \(d^{(l)}\in\mathbb{R}^{n_l}_{+}\), and define
the loop-transformed activation
\[
\psi^{(l)}(z;d^{(l)}) := \sigma^{(l)}(z) - D^{(l)}z.
\]
We also introduce the diagonal matrix
\[
D'^{(l)} := \mathrm{diag}\!\big(\max\{|\beta^{(l)}-d^{(l)}|,\;|d^{(l)}-\alpha^{(l)}|\}\big),
\]
With these definitions, the network admits the following equivalent representation:
\begin{align}\label{eq:nnDefLT}
z^{(l)}(x) \! &= \! W^{(l)}a^{(l-1)}(x),\notag \\ \notag
a^{(l)}(x) \! &= \! \psi^{(l)}\!\big(z^{(l)}(x);d^{(l)}\big) \!+\! D^{(l)}z^{(l)}(x),
\ l=1,\ldots,L\!-\!1,\\
f(x) \! &= \! z^{(L)}(x) = W^{(L)}a^{(L-1)}(x),
\end{align}
with \(a^{(0)}(x)=x\).
For convenience, we will write \(\mathcal{D}^{(l)} := \{D^{(1)},\ldots,D^{(l)}\}\) to denote the collection
of loop-transformation matrices up to layer \(l\).
In the next theorem, we use \eqref{eq:nnDefLT} to derive a parameterized (and optimizable) upper bound
on the Lipschitz constant of \(f\).



\begin{theorem}[Layerwise LipLT bound]\label{thm:LocalLipLT} 
Consider the loop-transformed multi-layer representation \eqref{eq:nnDefLT} together with the
auxiliary outputs
\[
y^{(l)}(x) := D'^{(l)} z^{(l)}(x), \qquad l=1,\ldots,L-1,
\]
and \(y^{(L)}(x):=z^{(L)}(x)=f(x)\).
Assume that each activation satisfies \(\sigma^{(l)}_i \in \mathrm{slope}(\alpha^{(l)}_i,\beta^{(l)}_i)\),
and choose loop parameters such that $0 \le d^{(l)}_i \le (\alpha^{(l)}_i+\beta^{(l)}_i)/2$ for $i=1,\ldots,n_l,\;\; l=1,\ldots,L-1$. 
Under these constraints, the residual slope bound simplifies elementwise to
\[
\max\{|\beta^{(l)}-d^{(l)}|,\;|d^{(l)}-\alpha^{(l)}|\}=\beta^{(l)}-d^{(l)},
\]
and hence \(D'^{(l)}:=\mathrm{diag}(\beta^{(l)}-d^{(l)})\) for \(l=1,\ldots,L-1\), and
\(D'^{(L)}:=I\). Define \(m^{(1)}(\mathcal{D}^{(1)}) := \|D'^{(1)}W^{(1)}\|_{p}\). For \(l\ge 2\), define recursively
\begin{align}\label{eq:LoopTransLip}
& m^{(l)}(\mathcal{D}^{(l)})
:= \Bigl\|D'^{(l)}W^{(l)}\prod_{i=1}^{l-1}\!\bigl(D^{(i)}W^{(i)}\bigr)\Bigr\|_{p} \notag\\
& + \sum_{j=1}^{l-1}
\Bigl\|D'^{(l)}W^{(l)}\prod_{i=j+1}^{l-1}\!\bigl(D^{(i)}W^{(i)}\bigr)\Bigr\|_{p}\;
m^{(j)}(\mathcal{D}^{(j)}),
\end{align}
where the empty product (when \(j=l-1\)) is the identity. Then, for each \(l=1,\ldots,L\), the map \(x\mapsto y^{(l)}(x)\) is Lipschitz continuous with Lipschitz
constant \(m^{(l)}(\mathcal{D}^{(l)})\). In particular, since \(y^{(L)}(x)=f(x)\) and \(D'^{(L)}=I\),
\(m^{(L)}(\mathcal{D}^{(L)})\) is a Lipschitz constant for \(f\).
\end{theorem}
%
%

\begin{proof}
    Consider two arbitrary inputs $x,\tilde{x}\in\mathcal{X}$. We suppress the input dependence of the network variables when no ambiguity arises.
    Consider the loop-transformed representation in \eqref{eq:nnDefLT}. Then
    \begin{align}
        y^{(l)} = D'^{(l)}W^{(l)}a^{(l-1)} 
        = &D'^{(l)}W^{(l)}\psi^{(l-1)} \notag \\
        + &D'^{(l)}W^{(l)}D^{(l-1)} z^{(l-1)}. \notag
    \end{align}
    Recursively applying loop transformation to terms that do not contain $\psi$, we get
    \begin{equation}
        \begin{split}\label{eq:defY}
            y^{(l)} 
            &= D'^{(l)}W^{(l)}\psi^{(l-1)} \\
            &+  D'^{(l)}W^{(l)} D^{(l-1)} W^{(l-1)} \psi^{(l-2)} + \cdots \\
            &+ D'^{(l)}W^{(l)} \prod_{i=l - k}^{l-1} (D^{(i)}W^{(i)}) \psi^{(l - k -1)} + \cdots\\
            & + D'^{(l)}W^{(l)} \prod_{i=2}^{l-1} (D^{(i)}W^{(i)}) \psi^{(1)} \\
            &+ D'^{(l)}W^{(l)} \prod_{i=1}^{l-1} (D^{(i)} W^{(i)}) x.
        \end{split}
    \end{equation}

    Similar to the proof of Lemma \ref{lemma:singlelaterLT}, we know that $z \mapsto \psi^{(l)}(z)$ satisfies the following
    \begin{align}\label{eq:psiSlope}
         \|\psi^{(l)}(z) - \psi^{(l)}(\tilde{z})\| 
        \leq \|D'^{(l)}(z - \tilde{z})\|.
    \end{align}
    %
    Using \eqref{eq:psiSlope}, we can bound the Lipschitz of $x \mapsto \psi^{(l)}(z^{(l)}(x))$ as follows
    \begin{align}
        \|\psi^{(l)}(z^{(l)}(x)) &- \psi^{(l)}(z^{(l)}(\tilde{x}))\| 
        \leq  \|D'^{(l)}(z^{(l)}(x)-z^{(l)}(\tilde{x}))\|  \notag \\
        &=  \|y^{(l)}(x) - y^{(l)}(\tilde{x})\|
    \end{align}
    where we have used the definition of $y^{(l)}$ from \eqref{eq:nnDefLT} in the second line.
    We now show by induction on $l$ that $m^{(l)}(\mathcal{D}^{(l)})$, as defined in \eqref{eq:LoopTransLip}, is a Lipschitz bound for $y^{(l)}$. For $l=1$, the claim follows directly from the definition of $m^{(1)}(\mathcal{D}^{(1)})$.

    Suppose the claim holds for all $j<l$. By \eqref{eq:psiSlope}, the Lipschitz constant of $x\mapsto\psi^{(j)}(z^{(j)}(x))$ is bounded by $m^{(j)}(\mathcal{D}^{(j)})$ for every $j<l$. 
    Using these bounds for the residual terms in \eqref{eq:defY} yields
    \begin{equation}
        \|y^{(l)}(x)-y^{(l)}(\tilde{x})\|
        \leq m^{(l)}(\mathcal{D}^{(l)})\|x-\tilde{x}\|, \notag
    \end{equation}
    where the resulting bound is precisely the recursion in
    \eqref{eq:LoopTransLip}.
\end{proof}

Algorithm~\ref{alg:LocalLipLT} computes the Lipschitz bounds from
Theorem~\ref{thm:LocalLipLT}. These bounds apply to the auxiliary subnetworks
\(x\mapsto y^{(l)}(x)=D'^{(l)}z^{(l)}(x)\). In the sequel, we will also need
Lipschitz bounds for the pre-activation subnetworks \(x\mapsto z^{(l)}(x)\).
The following corollary shows that such bounds can be obtained by modifying only
the final output multiplier in the recursion.

\begin{corollary}
\label{cor:lipLTsubnet}
Fix \(l\in\{1,\ldots,L\}\). Consider the recursion in
\eqref{eq:LoopTransLip} applied to the subnetwork
\[
x \mapsto z^{(l)}(x)=W^{(l)}a^{(l-1)}(x),
\]
with the same loop-transformation matrices
\(\mathcal{D}^{(l-1)}=\{D^{(1)},\ldots,D^{(l-1)}\}\) for the preceding hidden layers,
but with the final multiplier set to
\(
D'^{(l)} := I.
\)
Then the resulting quantity, denoted by \(\widehat m^{(l)}(\mathcal{D}^{(l-1)})\), is a
Lipschitz constant for the map \(x\mapsto z^{(l)}(x)\).
\end{corollary}

\begin{proof}
The proof follows directly from Theorem~\ref{thm:LocalLipLT}. In that theorem,
\(m^{(l)}\) bounds the Lipschitz constant of
\(
x\mapsto y^{(l)}(x)=D'^{(l)}z^{(l)}(x).
\)
Setting the final multiplier to \(D'^{(l)}=I\) yields
\(
y^{(l)}(x)=z^{(l)}(x),
\)
and therefore the same recursion yields a Lipschitz constant for
\(x\mapsto z^{(l)}(x)\).
\end{proof}





\paragraph*{Choice of loop transformation.}
Theorem~\ref{thm:LocalLipLT} provides a family of analytical Lipschitz upper bounds
parameterized by the loop-transformation matrices
\(\mathcal{D}^{(L-1)}=\{D^{(1)},\ldots,D^{(L-1)}\}\).
When \(D^{(l)}=0\) for all hidden layers, the resulting bound reduces to the naive
layerwise Lipschitz bound. Hence, in principle, the bound can be tightened by solving
\begin{align}\label{eq:multi_layer_liplt_opt}
\min_{\{D^{(l)}\}_{l=1}^{L-1}}
\ 
m^{(L)}(\mathcal{D}^{(L)})
\end{align}
subject to the admissibility constraints on the loop parameters. However, this optimization
problem can be nonconvex and may introduce additional computational overhead. We therefore
use a simple analytical choice that provably improves the naive bound.

\begin{proposition}[Analytical loop-transformation]
\label{prop:LT}
Assume \(0\le \alpha^{(l)}\le \beta^{(l)}\) elementwise for all hidden layers
\(l=1,\ldots,L-1\). Let
\[
D^{(l)}=\mathrm{diag}\!\left(\frac{\beta^{(l)}}{2}\right),
\qquad l=1,\ldots,L-1.
\]
Then the resulting multi-layer \LipMethodName bound satisfies
\begin{align}
m^{(L)}(\mathcal{D}^{(L)})
\le
m^{(L)}(0),
\end{align}
where \(m^{(L)}(0)\) denotes the naive bound obtained by setting \(D^{(l)}=0\) for all
hidden layers.
\end{proposition}

The proof is provided in Appendix~\ref{subsec:proofs}.



\begin{remark}[Comparison with global \LipMethodName~\cite{fazlyab2023certified}]
\label{rem:local_vs_global_liplt}
Global \LipMethodName can be recovered as a special case of the localized formulation by
relaxing all unit- and layer-dependent slope bounds to a single global interval
\[
\tilde{\alpha}:=\min_{l,i}\alpha^{(l)}_i,
\qquad
\tilde{\beta}:=\max_{l,i}\beta^{(l)}_i.
\]
With these homogeneous bounds, choosing the shared loop transformation
\[
d^{(l)}=\frac{\tilde{\alpha}+\tilde{\beta}}{2}\mathbf{1}_{n_l},
\qquad l=1,\ldots,L-1,
\]
recovers the global \LipMethodName construction of~\cite{fazlyab2023certified}.
Thus, the localized method can be viewed as a refinement of global \LipMethodName that exploits
tighter, layer- and unit-specific slope information over the input region of interest.
\end{remark}

Having established scalable bounds on the Lipschitz constant, we next derive bounds on the Hessian
of neural networks.



\begin{algorithm}[t]
   \caption{Local LipLT}
   \label{alg:LocalLipLT}
\begin{algorithmic}
   \STATE {\bfseries Input:} $L$-layer neural network in the form of \eqref{eq:NNdef}. \\
   Loop transformation matrices $D^{(l)}$ for $l = 1, \cdots L-1$.
   \STATE {\bfseries Output:} Lipschitz constant of the network $m^{(L)}(\mathcal{D}^{(L)})$.
   \STATE {\bfseries Initialize} $m^{(1)}(\mathcal{D}^{(1)}) = \|D'^{(1)}W^{(1)}\|$
   \FOR{$l=2$ {\bfseries to} $L$}
   \STATE Calculate $m^{(l)}(\mathcal{D}^{(l)})$ according to \eqref{eq:LoopTransLip}.
   \ENDFOR
   \STATE {\bfseries Return} $m^{(L)}(\mathcal{D}^{(L)})$.
\end{algorithmic}
\end{algorithm}

\section{Hessian bounds for smooth neural networks}
\label{sec:hess}
In this section, we leverage local \LipMethodName{} to construct
the matrices $M$ and $N$ that bound the Hessian of $J$
in \eqref{eq: bounds on Hessian}.
Recall that the scalar objective $J(x)=c^\top f(x)$
defined in \eqref{eq:boundProblem} can be represented
as a neural network by modifying only the final affine
layer of $f$:
\begin{align*}
    \tilde{W}^{(L)}=c^\top W^{(L)},
    \qquad
    \tilde{b}^{(L)}=c^\top b^{(L)}.
\end{align*}
Thus, $J:\mathbb{R}^{n_x}\to\mathbb{R}$ inherits the
architecture and hidden layers of $f$, with a scalar
output layer. Figure~\ref{fig:description} illustrates
the resulting network architecture.



\subsection{Two-layer neural networks}
\label{subsec:Two-layer neural networks}

Consider the two-layer neural network
$J(x)=c^\top W^{(2)}$ $\sigma(W^{(1)}x)$.
Its Hessian is given by
\begin{align}\label{eq: hessian two-layer}
    \nabla^2 J(x)
    &= W^{(1)\top}
    \mathrm{diag}\left(
        W^{(2)\top}c\odot\sigma''(W^{(1)}x)
    \right)W^{(1)} \notag\\
    &= \sum_{j=1}^{n_1}
    (c^\top W^{(2)})_j
    \sigma_j''(W_j^{(1)\top}x)
    W_j^{(1)}W_j^{(1)\top}.
\end{align}
Suppose that, for each hidden unit $j\in[n_1]$,
the second derivative of its activation function
satisfies the local bounds
\begin{align}
    \alpha'_j
    \leq \sigma_j''(W_j^{(1)\top}x)
    \leq \beta'_j,
    \qquad \forall x\in\mathcal{X}.
    \notag
\end{align}
Equivalently, the derivative $\sigma_j'$ is
slope-restricted on the corresponding pre-activation
interval, with bounds $\alpha'_j$ and $\beta'_j$.
These constants can be obtained by first bounding
the pre-activation ranges using methods such as
IBP~\cite{gowal2018effectiveness} or
CROWN~\cite{zhang2018efficient}, and then bounding
$\sigma_j''$ over the resulting intervals.
By substituting the preceding bounds in \eqref{eq: hessian two-layer}, one can obtain the upper bound matrix $M$ as follows,
\begin{align}
    &M = \sum_{j=1}^{n_1}  M_j W^{(1)}_jW^{(1)\top}_j, \notag \\
    &M_j = \Bigl( \beta'_{j}\max(0, (c^\top W^{(2)})_{j}) + \alpha'_{j} \min (0, (c^\top W^{(2)})_{j}) \Bigr). \notag
\end{align}
The lower bound $N$ can similarly be computed as 
\begin{align}
    &N = \sum_{j=1}^{n_1} N_jW^{(1)}_jW^{(1)\top}_j,\notag \\
    &N_j = \Bigl( \alpha'_{j}\max(0, (c^\top W^{(2)})_{j}) + \beta'_{j} \min (0, (c^\top W^{(2)})_{j}) \Bigr). \notag
\end{align}

\begin{figure}[t]
    \centering
    \includegraphics[width=0.5\textwidth]{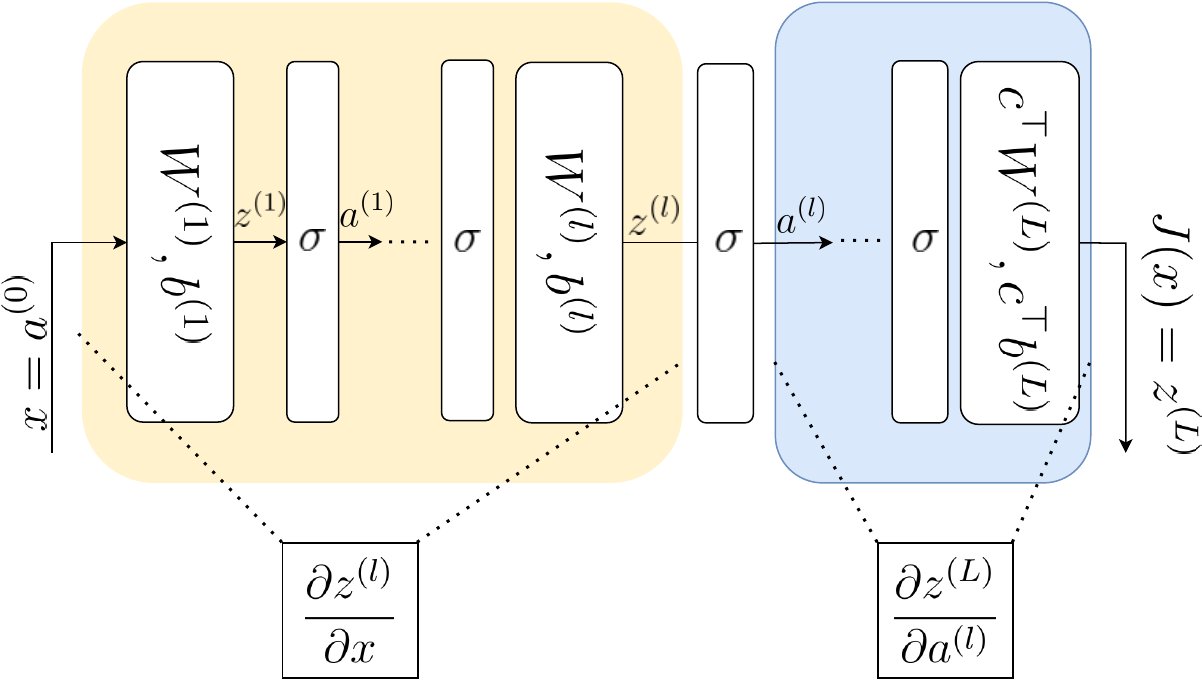}
    \caption{Architecture of the modified neural network. The yellow part is the sub-network consisting of the first $l$ layers, and the blue part is the sub-network consisting of the final $(L-l)$ layers.}
    \label{fig:description}
\end{figure}

\subsection{Multi-layer neural networks}
For multi-layer neural networks, deriving explicit matrix-valued
Hessian bounds $M$ and $N$, as in the two-layer case,
is more challenging. Instead, we bound the spectral norm
of the Hessian and construct the symmetric matrix bounds
\begin{align}
    M=-N=\gamma I,
    \qquad
    \gamma\geq\sup_{x\in\mathcal{X}}\|\nabla^2J(x)\|_2.
    \notag
\end{align}
Although potentially conservative, these bounds are
computationally tractable and applicable to general
multi-layer architectures. Our derivation builds on
the following result from~\cite{singla2020second}.
\begin{lemma}{\cite[Lemma 1]{singla2020second}}\label{lemma:hessian}
Consider an $L$-layer scalar-valued neural network
$J(x)=z^{(L)}(x)$ of the form \eqref{eq:NNdef}.
Its Hessian with respect to the input $x$ is given by 
    \begin{align}
        \nabla^2 z^{(L)} = 
        \sum_{l=1}^{L-1} \frac{\partial z^{(l)}}{\partial x}^\top \mathrm{diag}(\sigma''^{(l)}(z^{(l)}) \odot ({\frac{\partial z^{(L)}}{\partial a^{(l)}}})^\top )\frac{\partial z^{(l)}}{\partial x}. \notag
    \end{align}
\end{lemma}
%
Taking the spectral norm of the Hessian expression
in Lemma~\ref{lemma:hessian} and applying the triangle
inequality and submultiplicativity yields
\begin{align}
    \|\nabla^2 J(x)\|_2
    \!\leq \! \sum_{l=1}^{L-1}
    \left\|\frac{\partial z^{(l)}}{\partial x}\right\|_2^2
    \max_{j\in[n_l]}
    \left|
        \sigma_j^{\prime\prime(l)}(z_j^{(l)})
        \left(
            \frac{\partial z^{(L)}}{\partial a^{(l)}}
        \right)_j
    \right|.
    \notag
\end{align}
Suppose $|\sigma''^{(l)}(z_{j}^{(l)})| \leq h_j^{(l)} := \max (|\alpha'^{(l)}_j|, |\beta'^{(l)}_j|)$ for some constant $h_j^{(l)}>0$; similar to the two-layer case, $h_j^{(l)}$ can be found using bound propagation. Then we can write
\begin{align}\label{eq:hessianNorm}
    \|\nabla^2 J(x)\|_2 
    &\leq  \sum_{l=1}^{L-1} \|\frac{\partial z^{(l)}}{\partial x}\|_2^2 \max_j (|h_j^{(l)} {(\frac{\partial z^{(L)}}{\partial a^{(l)}})_j}|) \\
    & = \sum_{l=1}^{L-1} \|\frac{\partial z^{(l)}}{\partial x}\|_2^2 \cdot \|\mathrm{diag}(h^{(l)})  (\frac{\partial z^{(L)}}{\partial a^{(l)}})^\top  \|_{\infty} \notag
\end{align}
%
We now proceed to bound the individual terms on the right-hand side of \eqref{eq:hessianNorm}.

\paragraph*{Bounding
$\boldsymbol{\left\|\frac{\partial z^{(l)}}{\partial x}\right\|_2}$:}
By the chain rule, the input Jacobian of the
$l$-th pre-activation satisfies
\begin{align}
    \frac{\partial z^{(l)}}{\partial x}
    =
    \begin{cases}
        W^{(1)}, & l=1,\\[2mm]
        W^{(l)}
        \mathrm{diag}\bigl(\sigma^{\prime(l-1)}(z^{(l-1)})\bigr)
        \frac{\partial z^{(l-1)}}{\partial x},
        & l\geq 2.
    \end{cases}
    \notag
\end{align}
Applying submultiplicativity recursively yields a
conservative product of layerwise bounds, particularly
for deep networks, as in \cite{singla2020second}.
Instead, we exploit the fact that the spectral norm
$\|\frac{\partial z^{(l)}}{\partial x}\|_2$
is bounded by any valid $\ell_2$-Lipschitz constant of
the subnetwork $x\mapsto z^{(l)}(x)$ over $\mathcal{X}$.
Thus, we apply local \LipMethodName{}
(see Corollary~\ref{cor:lipLTsubnet}) to obtain bounds
for all pre-activation subnetworks
$x\mapsto z^{(l)}(x)$, $l=1,\ldots,L-1$,
required in \eqref{eq:hessianNorm}.

\paragraph*{Bounding
$\boldsymbol{
\|
\mathrm{diag}(h^{(l)})
\left(\frac{\partial z^{(L)}}{\partial a^{(l)}}\right)^\top
\|_\infty
}$:}
The weighted gradient norm can be bounded using
the $\ell_1$-Lipschitz constant of a modified subnetwork
mapping $a^{(l)}$ to $z^{(L)}$.
Specifically, we incorporate the curvature weights
$h^{(l)}$ into the first layer of this subnetwork by setting
$W^{(l+1)}\gets W^{(l+1)}\mathrm{diag}(h^{(l)})$.
The resulting weighted gradient norm can then be
bounded using local \LipMethodName{} with $p=1$.
Alternatively, we adapt the bound propagation approach
of~\cite{singla2020second} to obtain an analytical
bound on the backward sensitivity.

\begin{proposition}\label{prop:lInfBound}
For the scalar-valued network $z^{(L)}$, we have
\begin{align}
    \left\|
    \left(\dfrac{\partial z^{(L)}}{\partial a^{(l)}}\right)^\top
    \right\|_\infty
    \leq \|S^{(L,l)}\|_1,
\end{align}
where $S^{(L,l)}$ is defined recursively as
\begin{align}\label{eq:infinitynorm}
    S^{(L,l)}
    =
    \begin{cases}
        |W^{(L)}|,
        & l=L-1,\\
        |W^{(L)}|
        \mathrm{diag}(\beta^{(L-1)})
        S^{(L-1,l)},
        & l\leq L-2.
    \end{cases}
\end{align}
\end{proposition}

See Appendix~\ref{subsec:proofs} for the proof.

Applying Proposition~\ref{prop:lInfBound} to the
modified subnetwork, with
$W^{(l+1)}\gets W^{(l+1)}\mathrm{diag}(h^{(l)})$,
provides the desired weighted gradient bound
through \eqref{eq:infinitynorm}.

\section{Reachability Analysis}\label{sec:PropMeth}
Having established local Hessian bounds $M$ and $N$ for the
scalar objective $J$, we now develop tractable solutions and
relaxations for the quadratic bounding problems in
\eqref{eq:SecondOrderLowerBound} over general $\ell_p$-norm balls.
We first consider Taylor expansions centered at $x_c$ and
then extend the analysis to arbitrary expansion points
$y\in\mathcal{X}$. Finally, we incorporate these bounds
into a branch-and-bound framework to approximate the
global optimum of \eqref{eq:boundProblem} to arbitrary accuracy.

Throughout this section, we focus on the case
$\lambda_{\max}(M)\geq 0$ and $\lambda_{\min}(N)\leq 0$,
which naturally arises from the symmetric Hessian bounds
$M=-N=\gamma I$ developed in $\S$\ref{sec:hess}.
The remaining cases can be addressed using standard
quadratic optimization techniques.

\subsection{Computing First-Order Bounds}
\label{subsec:first_order_upper_bounds}

We begin by deriving a tractable first-order upper bound
using only the scalar curvature bound $\lambda_{\max}(M)$,
rather than the full matrix $M$. When the full matrix $M$ is available, as in the case of
two-layer networks, tighter bounds can be obtained by
directly exploiting its structure.
We discuss this case in Appendix~\ref{sec:two_layer_net}.

\begin{proposition}
\label{thm:operatorNorm}
Suppose \(M = \lambda_{\max}(M) I_{n_x}\) in \eqref{eq:subeq2} with 
\(\lambda_{\max}(M)\ge 0\). Then, for
\(\mathcal{X}=\{x_c+\delta \mid \|\delta\|_p\le \varepsilon\}\), the first-order upper bound satisfies
\begin{align}
\overline{J}^1_{\mathcal{X}}(x_c) \! \leq \!
        J(x_c) \! + \! \|\nabla J(x_c)\|_{p^*}\varepsilon \! + \!
        \frac{\lambda_{\max}(M)}{2} \! n_0^{\max(0, 1-\frac{2}{p})} \varepsilon^2 \!. \notag
\end{align}
Furthermore, the inequality is tight for $p \in \{2,\infty\}$.
\end{proposition}

   
    \begin{proof} 
        Let $y = x_c$, and define $\delta = x - x_c$. We can write
        \begin{align*}
            &\overline{J}^1_{\mathcal{X}}(x_c) =  J(x_c) +\sup_{\|\delta\|_p \leq \varepsilon} \{ \nabla J(x_c)^\top \delta + \frac{\lambda_{\max}(M)}{2} \|\delta\|_{2}^2 \} \notag\\
            & \leq J(x_c)\! +\!\sup_{\|\delta\|_p \leq \varepsilon} \{ \nabla J(x_c)^\top \delta \} \!+\! \sup_{\|\delta\|_p \leq \varepsilon} \frac{\lambda_{\max}(M)}{2} \|\delta\|_2^2 \notag \\
            &\leq J(x_c) + \|\nabla J(x_c)\|_{p^*}\varepsilon +\frac{\lambda_{\max}(M)}{2} n_0^{\max(0, 1-\frac{2}{p})}  \varepsilon^2. \notag
        \end{align*}
        In the last inequality, we have used the definition of the dual norm and the following
        \begin{align}
            \sup_{\|\delta\|_p \leq \varepsilon} \|\delta\|_2^2 \leq
            \sup_{\|\delta\|_p \leq  \varepsilon} n_0^{\max(0, 1-\frac{2}{p})} \|\delta\|_p^2 = 
            n_0^{\max(0, 1-\frac{2}{p})} \varepsilon^2, \notag
        \end{align}
        which can be shown using \eqref{eq:normIneq}.
        For $p=2$, the bound is exact: if
        $\nabla J(x_c)\neq 0$, the choice
        $\delta^\star=\varepsilon
        \nabla J(x_c)/\|\nabla J(x_c)\|_2$
        attains equality. If $\nabla J(x_c)=0$,
        any $\delta^\star$ satisfying
        $\|\delta^\star\|_2=\varepsilon$
        is optimal. For $p=\infty$, the bound is also exact. In particular,
        choosing
        \(
        \delta^\star=\varepsilon\,
        \operatorname{sign}(\nabla J(x_c)),
        \)
        with either sign assigned to zero-gradient components,
        simultaneously maximizes the linear and quadratic terms,
        yielding
        \(
        \overline{J}_{\mathcal X}^1(x_c)
        =J(x_c)+\varepsilon\|\nabla J(x_c)\|_1
        +\frac{\lambda_{\max}(M)}{2}n_0\varepsilon^2.
        \)
%
        
    \end{proof}
    \subsection{Computing First-order Lower Bounds}
        We next consider the first-order lower bound in
        \eqref{eq:subeq1}. When only a scalar curvature bound
        is available, we take $N=\lambda_{\min}(N)I_{n_x}$,
        with $\lambda_{\min}(N)<0$.
        The resulting quadratic maximization problem is concave and can be solved efficiently using convex optimization.
        To avoid numerical optimization, we instead derive an
        analytical feasible point that provides a certified
        lower bound.
        

        \begin{proposition}
        \label{prop:firstOrderLower}
        Suppose $N=\lambda_{\min}(N)I_{n_x}$ in
        \eqref{eq:subeq1}, with $\lambda_{\min}(N)<0$.
        Let $g=\nabla J(x_c)$ and suppose $g\neq 0$.
        Define
        $
            r^*=
            \min\left\{
                -\frac{1}{\lambda_{\min}(N)},
                \frac{\varepsilon}{\|g\|_p}
            \right\}$.
        Then
        \begin{align}
            \underline{J}^1_{\mathcal{X}}(x_c)
            &\geq J(x_c)
            +r^*\|g\|_2^2
            \left(1+\frac{\lambda_{\min}(N)}{2}r^*\right)
            \notag\\
            &\geq \underline{J}^0_{\mathcal{X}}(x_c)
            =J(x_c).
        \end{align}
        If $g=0$, the same conclusion holds by choosing
        $x=x_c$.
        \end{proposition}

        \begin{proof}
        For $g\neq 0$, consider the feasible perturbation $\delta=r^*g$. By construction, $\|\delta\|_p=r^*\|g\|_p\leq\varepsilon.$
        Substituting this perturbation into \eqref{eq:subeq1} yields
        \begin{align*}
            \underline{J}^1_{\mathcal{X}}(x_c)
            &\geq J(x_c)+g^\top\delta
            +\frac{\lambda_{\min}(N)}{2}\|\delta\|_2^2\\
            &=J(x_c)+r^*\|g\|_2^2
            \left(1+\frac{\lambda_{\min}(N)}{2}r^*\right).
        \end{align*}
        Since $0\leq r^*\leq-1/\lambda_{\min}(N)$,
        we have
        $\underline{J}^1_{\mathcal{X}}(x_c)
        \geq J(x_c)
        =\underline{J}^0_{\mathcal{X}}(x_c)$. If $g=0$, choosing $\delta=0$ gives the same
        inequality.
        \end{proof}

\begin{figure}[t]
    \centering
    \includegraphics[width= 0.4\textwidth]{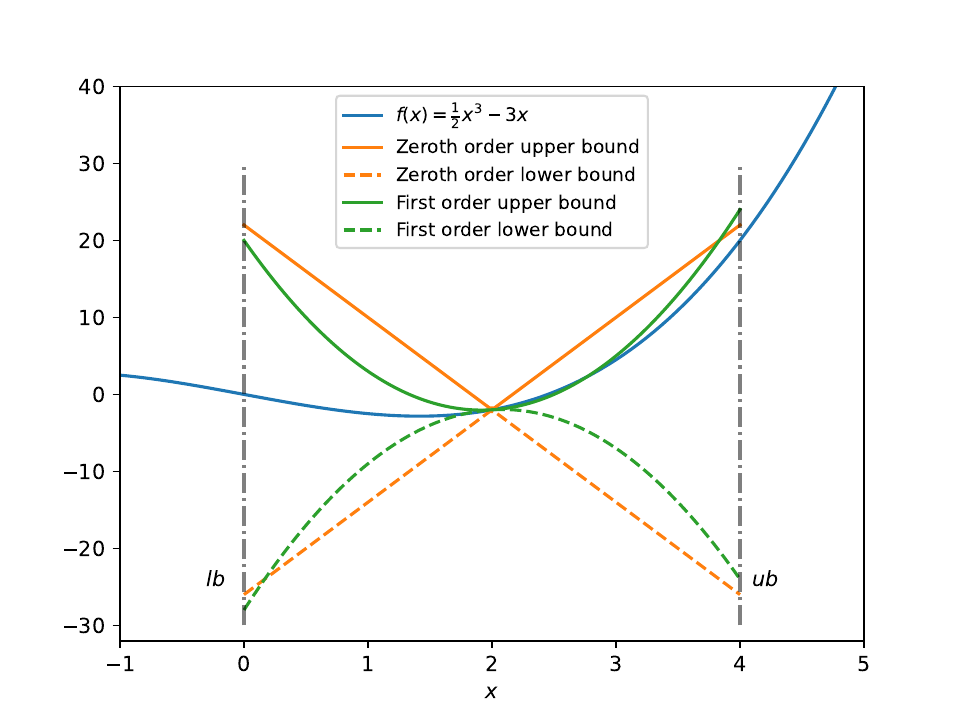}
    \caption{Comparison of zeroth- and first-order bounds.}
    \label{fig:method}
\end{figure}

\subsection{First-Order Versus Zeroth-Order Bounds}

We now compare the zeroth- and first-order bounds.
By incorporating local gradient and curvature information,
the first-order approximation can yield tighter upper bounds,
particularly over sufficiently small neighborhoods of $x_c$.
Moreover, the first-order lower bound is always at least
as tight as its zeroth-order counterpart.
Figure~\ref{fig:method} illustrates these properties.
The following theorem formalizes the comparison and
establishes a sufficient condition under which the
first-order upper bound is guaranteed to be no worse
than the zeroth-order bound.
    
    \begin{theorem}\label{prop:bndCmp}
Let $L$ be a Lipschitz constant of $J$ over
$\mathcal{X}=\{x\mid \|x-x_c\|_p\leq\varepsilon\}$,
and let $M$ be an upper bound on its Hessian as in~\eqref{eq: bounds on Hessian}. Then the following hold:
\begin{enumerate}
    \item If $\lambda_{\max}(M)>0$, then
    \(
        \overline{J}^1_{\mathcal{X}}(x_c)
        \leq \overline{J}^0_{\mathcal{X}}(x_c)
    \)
    whenever $0<\varepsilon\leq\varepsilon_{\max}$, where
    \[
        \varepsilon_{\max}
        =\frac{2n_0^{-\max\{0,1-2/p\}}}
        {\lambda_{\max}(M)}
        \left(L-\|\nabla J(x_c)\|_{p^*}\right).
    \]
    If $\lambda_{\max}(M)\leq 0$, the same inequality
    holds for every $\varepsilon\geq 0$.

    \item The first-order lower bound satisfies
    \[
        \underline{J}^1_{\mathcal{X}}(x_c)
        \geq \underline{J}^0_{\mathcal{X}}(x_c) \quad \forall \varepsilon\geq 0.
    \]
\end{enumerate}
\end{theorem}

\begin{proof}
Let $g=\nabla J(x_c)$ and
$\kappa=n_0^{\max\{0,1-2/p\}}$.
Suppose $\lambda_{\max}(M)>0$.
By Proposition~\ref{thm:operatorNorm}, we have
\begin{align*}
\overline J_{\mathcal X}^{1}(x_c)
&\leq J(x_c)+\|g\|_{p^*}\varepsilon
+\frac{\lambda_{\max}(M)}{2}\kappa\varepsilon^2.
\end{align*}
For $0<\varepsilon\leq\varepsilon_{\max}$,
the definition of $\varepsilon_{\max}$ implies
\[
\|g\|_{p^*}\varepsilon
+\frac{\lambda_{\max}(M)}{2}\kappa\varepsilon^2
\leq L\varepsilon.
\]
Consequently,
\(
\overline J_{\mathcal X}^{1}(x_c)
\leq J(x_c)+L\varepsilon
=\overline J_{\mathcal X}^{0}(x_c).
\) If $\lambda_{\max}(M)\leq0$, then $M\preceq0$,
and hence
\begin{align*}
\overline J_{\mathcal X}^{1}(x_c)
&\leq J(x_c)+
\sup_{\|\delta\|_p\leq\varepsilon}g^\top\delta\\
&=J(x_c)+\|g\|_{p^*}\varepsilon\\
&\leq J(x_c)+L\varepsilon
=\overline J_{\mathcal X}^{0}(x_c).
\end{align*}
Finally, since $x=x_c$  is feasible in~\eqref{eq:subeq1},
\(
\underline J_{\mathcal X}^{1}(x_c)
\geq J(x_c)
=\underline J_{\mathcal X}^{0}(x_c),
\)
which completes the proof.
\end{proof}

%

\paragraph*{Optimal First-order Bound}
In deriving the first-order bounds, we used a Taylor approximation of $J$ around the \emph{center} $x_c$ of $\mathcal{X}$ by setting $y=x_c$. However, this choice can be sub-optimal due to the asymmetry of the Taylor approximation. We can instead consider a Taylor approximation at an \emph{arbitrary} $y \in \mathcal{X}$ and optimize the resulting bounds over $y$. 
Explicitly, we have
\begin{align}
\underline{J}^1_{\mathcal{X}}(x_c) \! \leq \!
\sup_{y \in \mathcal{X}} \underline{J}^1_{\mathcal{X}}(y) \! \leq \! 
\sup_{x \in \mathcal{X}} J(x) \! \leq \!
\inf_{y \in \mathcal{X}} \bar{J}^1_{\mathcal{X}}(y)  \! \leq \! 
\bar{J}^1_{\mathcal{X}}(x_c).  \notag
\end{align}
In the following proposition, we first find the maximizer of the first-order upper bound \eqref{eq:subeq2} as a function of $y$. We will then prescribe an analytical $y \neq x_c$ that improves the bounds obtained by the default choice $y=x_c$.

\begin{proposition}\label{prop:firstOrderatY}
    Suppose $M=\lambda_{\max}(M)I_{n_x}$ in
    \eqref{eq:subeq2}, with $\lambda_{\max}(M)\geq 0$.
    Define
    \[
        a(y):=\nabla J(y)-\lambda_{\max}(M)(y-x_c).
    \]
    For $p=2$ and $a(y)\neq 0$, a maximizer of
    \eqref{eq:subeq2} is given by
    \begin{align}
        x^*(y)=x_c+\varepsilon\frac{a(y)}{\|a(y)\|_2}.
        \notag
    \end{align}
    If $a(y)=0$, any point satisfying
    $\|x^*(y)-x_c\|_2=\varepsilon$ is optimal.
    For $p=\infty$, a maximizer is given by
    \begin{align}
        x^*(y)=x_c+\varepsilon\,\mathrm{sign}(a(y)),
        \notag
    \end{align}
    where $\mathrm{sign}(0)=1$.
\end{proposition}

     See {Appendix \ref{subsec:proofs}} for the proof.

    Given $x^*(y)$, we aim to find the optimal $y$ by solving the optimization problem
    \begin{align}
        \inf_{y \in \mathcal{X}} J(y) \! + \! 
        \nabla J(y)^\top (x^*(y) - y) \! + \!  
        \frac{\lambda_{\max}(M)}{2} \|x^*(y) \! - \! y\|_2^2. \notag
    \end{align}
    This problem, however, is non-convex. In the following proposition, we provide an analytical $y$ that can provably improve the default bound $\bar{J}^1_{\mathcal{X}}(x_c)$. 

     \begin{proposition}\label{prop:heuristicY}
        Suppose
        \(
        \lambda_{\max}(M) \geq -\lambda_{\min}(N) > 0.
        \)
         Let $y=x_c + \eta \hat{\delta}$, where $\hat{\delta} = x^*(x_c) - x_c$ and $p = \infty$. 
         Then we have $\bar{J}^1_{\mathcal{X}}(y) \leq \bar{J}^1_{\mathcal{X}}(x_c)$ for every
         \begin{align}
             \eta \in [0, \min \big(\min_i \frac{|\nabla J(x_c)_i|}{\lambda_{\max}(M) (\|\hat{\delta}\|_2 + |\hat{\delta}_i|)}, 1 \big)]. \notag
         \end{align}

     \end{proposition}
     See {Appendix \ref{subsec:proofs}} for the proof.

\subsection{Integration with Branching}
In this section, we combine the proposed bounding method with a branching strategy, enabling us to solve \eqref{eq:boundProblem} within an arbitrary accuracy. 
Following \cite{stursberg2003efficient, entesari2023reachlipbnb}, we use a heuristic guided by Principal Component Analysis (PCA) to choose $\mathcal{C}$.

%
%
\begin{algorithm}[t]
   \caption{Branch-and-Bound Algorithm to solve \eqref{eq:boundProblem} for a given direction $c$ and input set $\mathcal{X}_0$}
   \label{alg:BnB}
\begin{algorithmic}
   \STATE {\bfseries Input:} 
   Input set $\mathcal{X}_0 = [\ell_0,u_0] \subset \mathbb{R}^n_0$, neural network $J \colon \mathbb{R}^{n_0} \to \mathbb{R}$, termination threshold $\varepsilon_t>0$. 
   \\
   \STATE {\bfseries Initialize:} $ub = \infty, lb = -\infty$ and 
   $\boldsymbol{\mathcal{X}}= \{\mathcal{X}_0$\}, \\
   \WHILE{$ub - lb > \varepsilon_t $}
   \STATE Choose $\mathcal{X} \in \boldsymbol{\mathcal{X}}$ according to \eqref{eq:branchHeuristic1}.\\
   \STATE Split $\mathcal{X}$ into $\mathcal{X}_{I}$ and $\mathcal{X}_{II}$ according to \eqref{eq:branchUB} or \eqref{eq:branchLA}.\\
   \STATE Replace $\mathcal{X}$ in $\boldsymbol{\mathcal{X}}$ with $\mathcal{X}_{I}$ and $\mathcal{X}_{II}$.\\
   \STATE Compute $\overline{J}_{\hat{\mathcal{X}}}$ and $\underline{J}_{\hat{\mathcal{X}}}$  for $\hat{\mathcal{X}} \in \{\mathcal{X}_{I},\mathcal{X}_{II}\}$ according to \eqref{eq:bnbBounding}. \\
   \STATE update $ub$, and $lb$ according to \eqref{eq:ubLb}.
   \ENDWHILE
\end{algorithmic}
\end{algorithm}
Algorithm \ref{alg:BnB} describes our branch-and-bound method. Given an input set $\mathcal{X}_0 = \lbrack \ell_0, u_0 \rbrack := \{x \mid \ell_0 \leq x\leq u_0\}$, the method recursively partitions $\mathcal{X}_0$ into disjoint sub-rectangles $\boldsymbol{\mathcal{X}} = \{\mathcal{X}_i\}_{i=1}^{N}$ and maintains a global lower and upper bound on $J_{\mathcal{X}_0}^*$ as
\begin{align}\label{eq:ubLb}
    ub &= \max_{\mathcal{X} \in \boldsymbol{\mathcal{X}}} \overline{J}_{\mathcal{X}}, \quad 
        lb = \max_{\mathcal{X} \in \boldsymbol{\mathcal{X}}} \underline{J}_{\mathcal{X}},
\end{align}
where $\underline{J}_\mathcal{X}$ and $\overline{J}_\mathcal{X}$ are any lower and upper bound on $\mathcal{X}$, respectively. In this paper, we have developed two bounding methods. We calculate both and choose the superior one. 
\begin{align}\label{eq:bnbBounding}
    \underline{J}_\mathcal{X} = \max(\underline{J}^0_\mathcal{X}, \underline{J}^1_\mathcal{X}), \quad
    \overline{J}_\mathcal{X} = \min(\overline{J}^0_\mathcal{X}, \overline{J}^1_\mathcal{X}).
\end{align}
We note that computing $\overline{J}^0_\mathcal{X}$ incurs no additional cost, as computing $\overline{J}^1_\mathcal{X}$ already subsumes all the required calculations. 
\blue{Moreover, based on Proposition \ref{prop:bndCmp}, as the algorithm progresses and the sub-rectangles shrink, the first-order bound eventually becomes smaller than the zeroth-order bound, speeding up the convergence of the algorithm towards the end.}

To refine the global bounds in \eqref{eq:ubLb}, the algorithm selects a sub-rectangle that admits the largest upper bound,
\begin{align}\label{eq:branchHeuristic1}
    \mathcal{X} = \arg \max_{\mathcal{X} \in \boldsymbol{\mathcal{X}}} \overline{J}_{\mathcal{X}}.
\end{align}
Then, it chooses a coordinate index $1 \leq j \leq n_0$, based on a given heuristic, to split the chosen subset $\mathcal{X}$ into two disjoint subsets $\mathcal{X}_{I}(j)$ and $\mathcal{X}_{II}(j)$ such that $\mathcal{X} = \mathcal{X}_{I}(j) \cup \mathcal{X}_{II}(j)$ and use the bounds on these new subsets to potentially improve the overall bound on the objective function. 
In this work, we use two strategies to choose which axis to split. The first strategy is choosing the dimension with the longest axis 
\begin{align}\label{eq:branchLA}
    j^* = \arg \max_j (u_j - \ell_j),
\end{align}
where $\mathcal{X} = \lbrack \ell, u \rbrack$ (See \cite{boyd2007branch} for the proof of convergence).
The second strategy involves splitting the set along all possible axes and selecting the one that yields the best upper bound \cite{bunel2018unified}.
\begin{align}\label{eq:branchUB}
    j^* = \argmin_j \max(\overline{J}_{\mathcal{X}_I}(j), \overline{J}_{\mathcal{X}_{II}}(j)).
\end{align}
 Following the standard arguments in branch-and-bound methods \cite{boyd2007branch, entesari2023reachlipbnb}, the lower bound $lb$ produces a non-decreasing sequence, and the upper bound $ub$ produces a non-increasing sequence on $J^*_{\mathcal{X}_0}$. 
The algorithm will terminate once it satisfies some termination criterion, such as $ub - lb \leq \varepsilon_t. $

\medskip \noindent

\smallskip \noindent
\textbf{Zonotope Input Sets:}
    The proposed methods in this paper rely on the input set $\mathcal{X}$ being a hyper-rectangle.
    However, we can also consider \textit{zonotopes} \cite{entesari2023automated}, which are defined as 
    $\mathcal{X} = \{Gz + x_c | \|z\|_\infty \leq 1\}$.
    The mapping from $z$ to $x$ can be interpreted as an additional affine layer with weight $G$ and bias $x_c$ that can be appended to the neural network $J$.  
    %
    Since $z$ is still restricted to a hyper-rectangle, we can apply our BnB framework by solving the equivalent problem $\max_{\|z\|_{\infty} \leq 1} J(Gz+x_c)$.
    

\begin{figure}[t!]
    \centering
    \includegraphics[width= 0.4\textwidth]{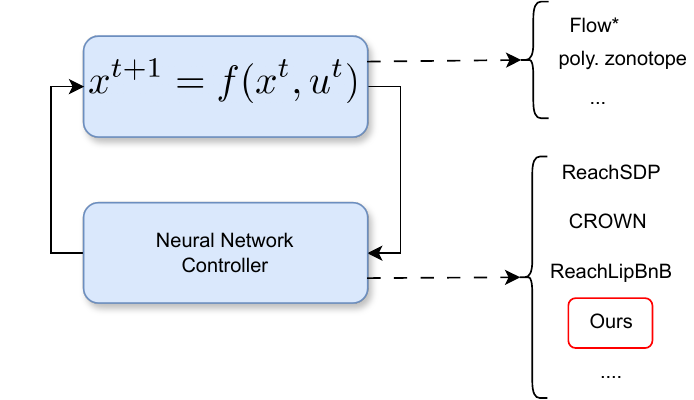}
    \caption{Closed-loop system reachability}
    \label{fig:closedloopreach}
\end{figure}
\smallskip
\begin{remark}
    The method discussed so far is tailored to perform reachability analysis on neural networks. To generalize this setting to reachability analysis of neural network-controlled systems, our method can be coupled with existing reachability algorithms for nonlinear systems--see Figure \ref{fig:closedloopreach}.
\end{remark}
\section{Experiments}\label{sec:NumRes}
We evaluate the proposed method against state-of-the-art
reachability analysis approaches across several benchmarks.
Following~\cite{entesari2023reachlipbnb}, we use principal
component analysis (PCA) on $10^4$ randomly sampled
trajectories propagated through the networks to determine
rotation matrices for non-axis-aligned set representations.
Unless otherwise specified, we report the mean and standard
deviation over five independent trials. For
ReachLipBnB~\cite{entesari2023reachlipbnb}, we impose a
memory limit of $2\times10^4$ active branches.
All experiments are conducted on an Intel Core i9-10980XE
processor with 64~GB of RAM. Our implementations and
comparisons build on publicly available code from
ReachLipBnB~\cite{entesari2023reachlipbnb}\footnote{
\url{https://github.com/o4lc/ReachLipBnB}},
\LipMethodName~\cite{fazlyab2023certified}\footnote{
\url{https://github.com/o4lc/CRM-LipLT}},
polynomial zonotopes~\cite{kochdumper2023open}\footnote{
\url{https://codeocean.com/capsule/8237552/tree/v1}},
and ReachMM~\cite{harapanahalli2023contraction}\footnote{
\url{https://github.com/gtfactslab/ReachMM_CDC2023}}.


\begin{figure*}[t!]
\centering

\begin{subfigure}{0.33\textwidth}
    \centering
    \includegraphics[width=0.9\columnwidth, trim={.5cm 0cm 1cm 0cm}]{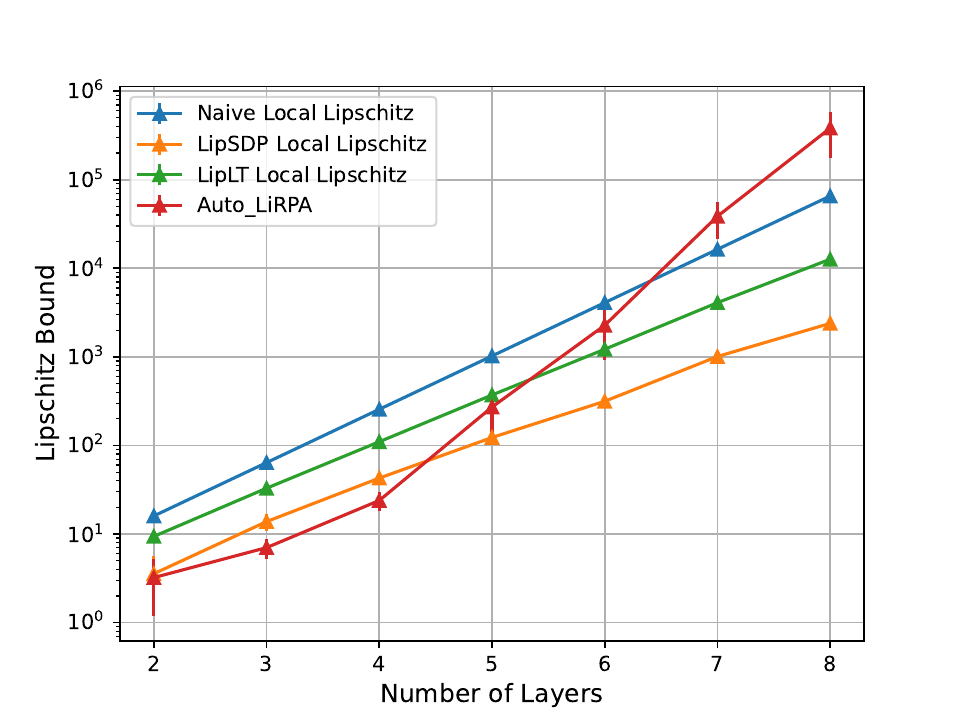}
\end{subfigure}%
\begin{subfigure}{0.33\textwidth}
    \centering
    \includegraphics[width=0.9\columnwidth, trim={.5cm 0cm 1cm 0cm}]{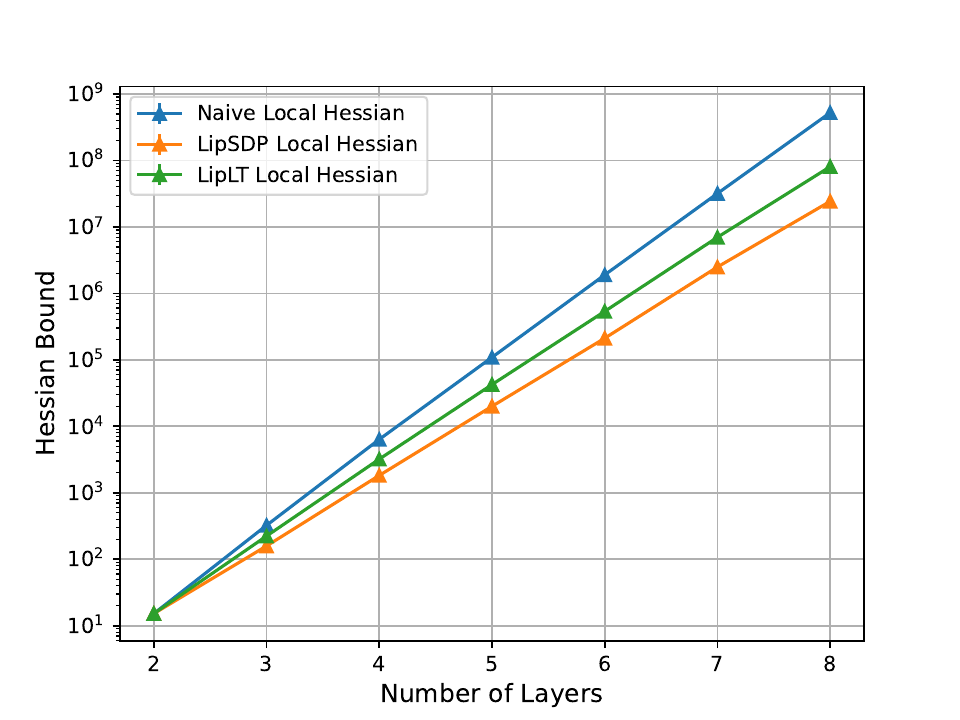}
\end{subfigure}%
\begin{subfigure}{0.33\textwidth}
    \centering
    \includegraphics[width=0.9\columnwidth, trim={.5cm 0cm 1cm 0cm}]{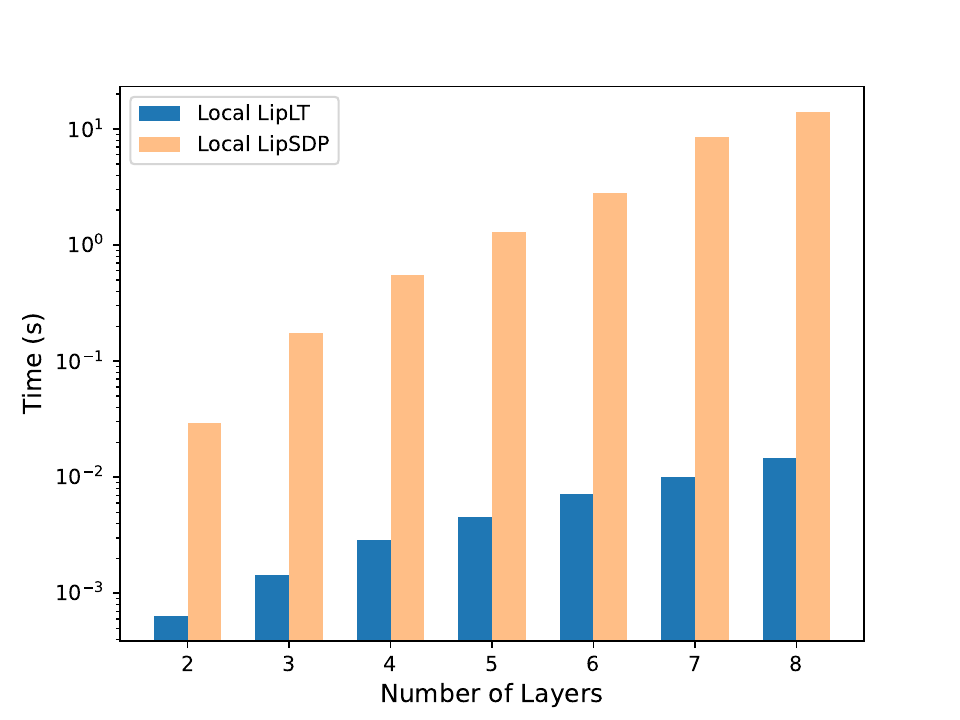}
\end{subfigure}
    \caption{\blue{(Left) Lipschitz and (middle) Hessian upper bound calculated by different methods and (right) time spent to calculate the local Hessian bounds per number of layers.}
    }
    \label{fig:hessianBound}
\end{figure*}


\subsection{Lipschitz and Hessian Bounds}
We first evaluate the tightness and computational efficiency
of the proposed local \LipMethodName{} approach for bounding
the Lipschitz constant and Hessian norm of neural networks.
We compare our method against LipSDP, global \LipMethodName{},
the local Lipschitz estimator implemented in
\texttt{Auto\_LiRPA}~\cite{shi2022efficiently}\footnote{
\url{https://github.com/Verified-Intelligence/Auto_LiRPA}},
and a naive baseline.
The experiments use randomly generated feedforward networks
with depths ranging from $2$ to $8$, each mapping
$\mathbb{R}^2$ to $\mathbb{R}$ with $32$ neurons per hidden layer.
Figure~\ref{fig:hessianBound} reports the resulting Lipschitz
and Hessian bounds, along with their computation times.
Overall, local \LipMethodName{} produces tighter bounds
than the naive baseline while offering substantially
lower computational cost than LipSDP.


\subsection{Open-Loop Reachability Analysis}

We compare our method with existing reachability approaches~\cite{singh2018fast,ivanov2021verisig,kochdumper2023open} on a small randomly initialized feedforward network with architecture \(2\times 50\times 2\), using either \(\tanh\) or sigmoid activations and the input set \([-1,1]^2\). We run our algorithm with two choices of output directions: standard basis directions and \(16\) uniformly spaced directions. Figure~\ref{fig:randNet} shows the reachable-set over-approximations produced by each method.

We further study the effect of depth using three randomly initialized \(\tanh\) networks with \(1\), \(2\), and \(3\) hidden layers; the corresponding reachable sets are shown in Figure~\ref{fig:randNetExp}. Both experiments follow the setup of~\cite{kochdumper2023open}. For our method, we set \(\varepsilon=10^{-2}\) and compute Lipschitz constants using local \LipMethodName. 

\blueTwo{
For these experiments (Figures~\ref{fig:randNet} and \ref{fig:randNetExp}), we also report runtime results in Table~\ref{tab:runtime}. Overall, even for more challenging tasks, our proposed method yields tight over-approximations (see Figure~\ref{fig:randNetExp}), at the expense of increased computational time.
}

\begin{figure}[t!]
    \centering
    \begin{subfigure}[t]{0.24\textwidth} 
    \includegraphics[width=0.99\textwidth, trim={0.5cm 0cm 0.5cm 0cm}]{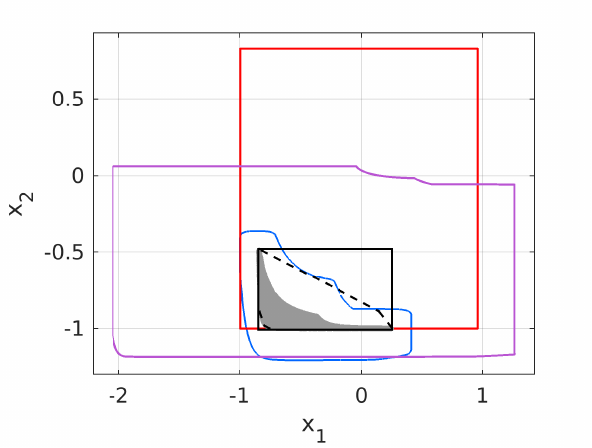}
     \end{subfigure}%
    \begin{subfigure}[t]{0.24\textwidth} \includegraphics[width=0.99\textwidth, trim={0.5cm 0cm 0.5cm 0cm}]{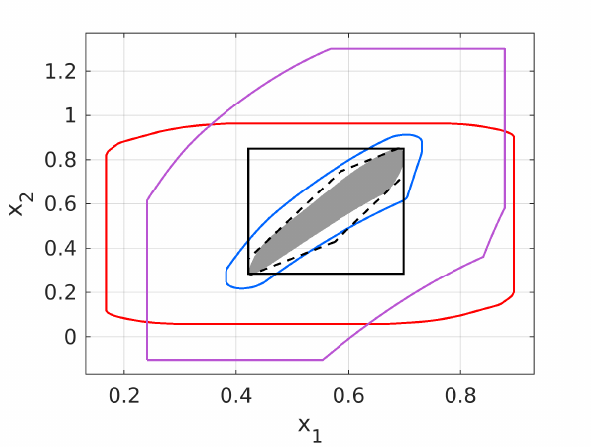}
     \end{subfigure}
    \caption{Reachable sets computed with zonotopes \cite{singh2018fast} (red), Taylor models \cite{ivanov2021verisig} (purple), polynomial zonotopes \cite{kochdumper2023open} (blue), our method with coordinate directions (solid black), and 16 uniformly chosen directions (dashed black) for two randomly initialized $\tanh$ (left) and Sigmoid (right) networks.}
    \label{fig:randNet}
\end{figure}

\begin{figure*}[t!]
    \centering
    \begin{subfigure}{0.33\textwidth}
        \centering
        \includegraphics[width=0.8\columnwidth, trim={0cm 0cm 1cm 0cm}]{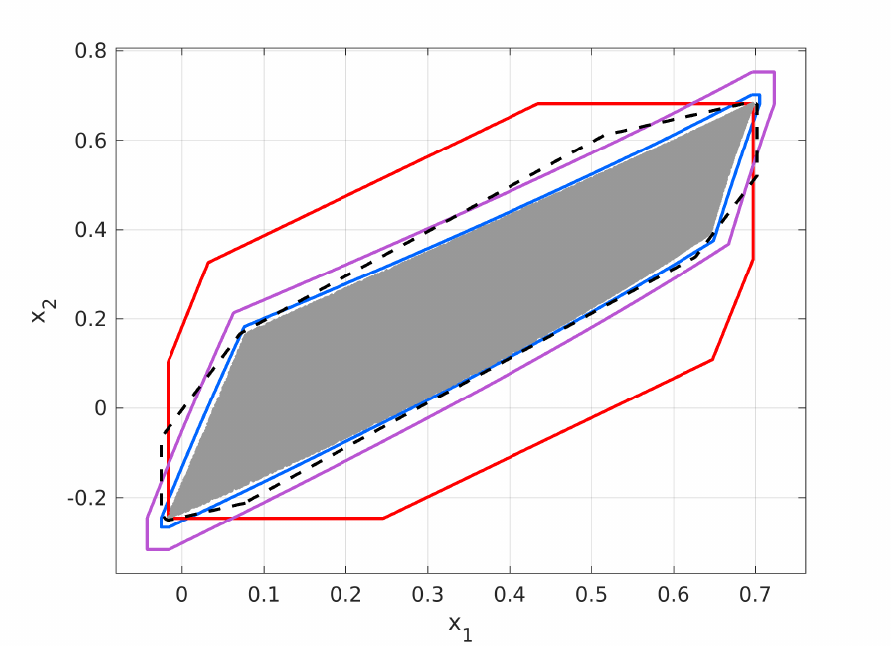}
    \end{subfigure}%
    \begin{subfigure}{0.33\textwidth}
        \centering
        \includegraphics[width=0.8\columnwidth, trim={0cm 0cm 1cm 0cm}]{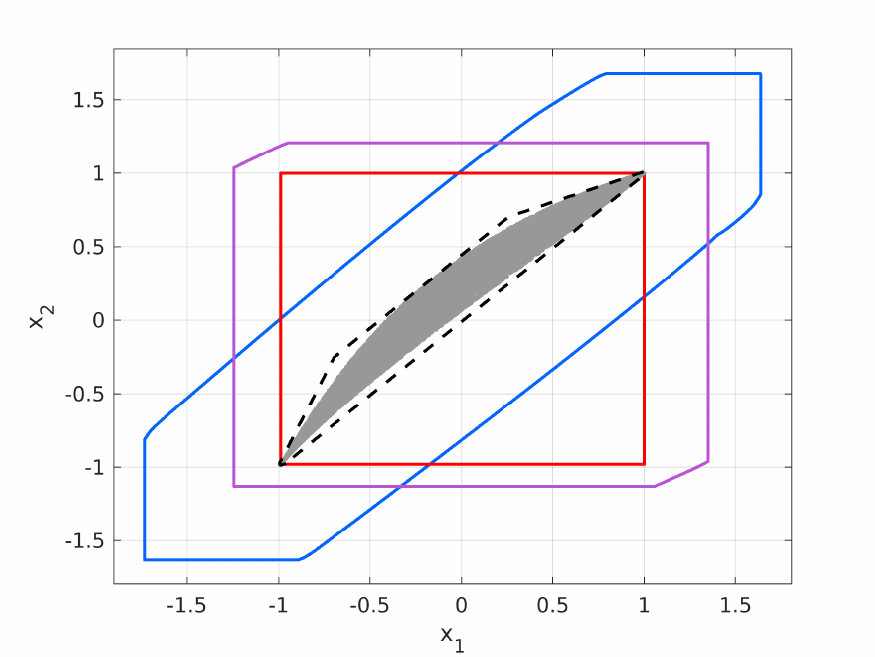}
    \end{subfigure}%
    \begin{subfigure}{0.33\textwidth}
        \centering
        \includegraphics[width=0.8\columnwidth, trim={0cm 0cm 1cm 0cm}]{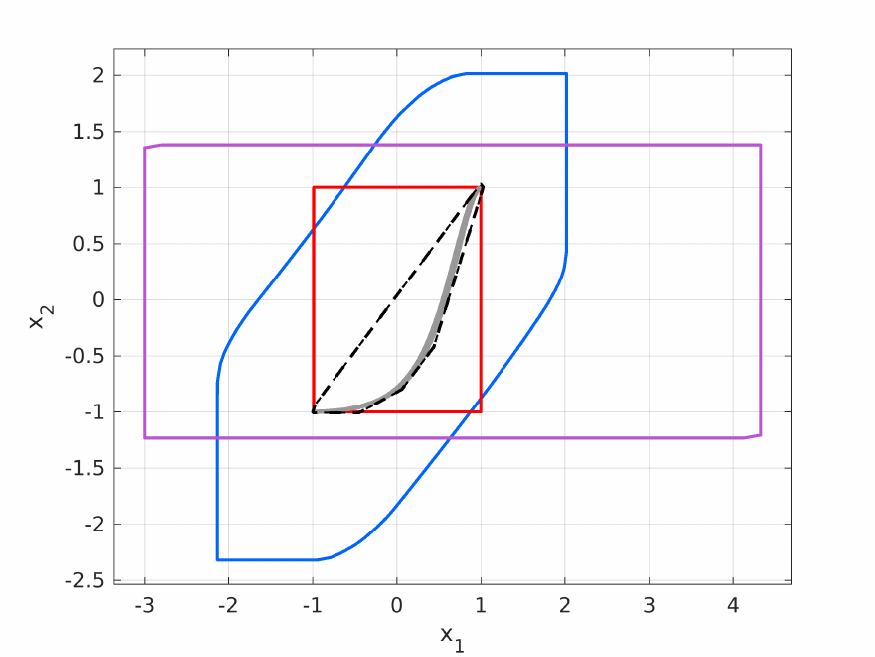}
    \end{subfigure}
    \caption{Reachable sets computed with zonotopes \cite{singh2018fast} (red), Taylor models \cite{ivanov2021verisig} (purple), polynomial zonotopes\cite{kochdumper2023open} (blue), our method with 16 uniformly chosen directions (dashed black) for a single layer (left), two-layer (middle), three-layer (right) randomly initialized $\tanh$ networks. 
    }
    \label{fig:randNetExp}
\end{figure*}

    \begin{table}[t]
        \centering
        \resizebox{0.95\linewidth}{!}{
        \begin{tabular}{l c c c c}
            \hline
            & Zono. \cite{singh2018fast} & Taylor \cite{ivanov2021verisig} & Poly. Zono. \cite{kochdumper2023open} & Ours \\
            \hline
            Sigmoid & 0.01 & 0.4 & 0.37 & 0.16 \\
            Tanh & 0.006 & 0.24 & 0.05 & 0.43 \\
            \hline
            Tanh 1-hidden & 0.002 & 0.02 & 0.01 & 0.05\\
            Tanh 2-hidden & 0.002 & 0.05 & 0.01 & 1.9\\
            Tanh 3-hidden & 0.003 & 0.13 & 0.17 & 47.5\\
            \hline
        \end{tabular}}
        \caption{Runtime comparison for Figs.~\ref{fig:randNet} and \ref{fig:randNetExp} .}
        \label{tab:runtime}
    \end{table}



\subsection{Closed-Loop Reachability Analysis}

We next consider the discrete-time closed-loop system
\begin{align}
x^{t+1}=f(x^t)=Ax^t+B u^t,
\qquad
u^t=\pi(x^t),
\end{align}
where \(x^t\in\mathbb{R}^{n_x}\) is the state, \(u^t\in\mathbb{R}^{n_u}\) is the control input generated by the neural network policy \(\pi\), and \(A\in\mathbb{R}^{n_x\times n_x}\), \(B\in\mathbb{R}^{n_x\times n_u}\) are the system matrices. Following~\cite{entesari2023reachlipbnb}, we absorb \(B\) into the final layer of \(\pi\), so that \(B\pi(x)\) is represented as a modified neural network. The closed-loop map can then be written as the sum of a linear term \(Ax\) and a neural-network term \(B\pi(x)\), allowing us to apply our reachability framework directly by computing the corresponding local Lipschitz and Hessian bounds.


 %
 \begin{figure}[t!]
\centering
\begin{subfigure}{0.24\textwidth}
    \centering
    \includegraphics[width=0.99\textwidth, trim={.5cm 0cm 1cm 0cm}]{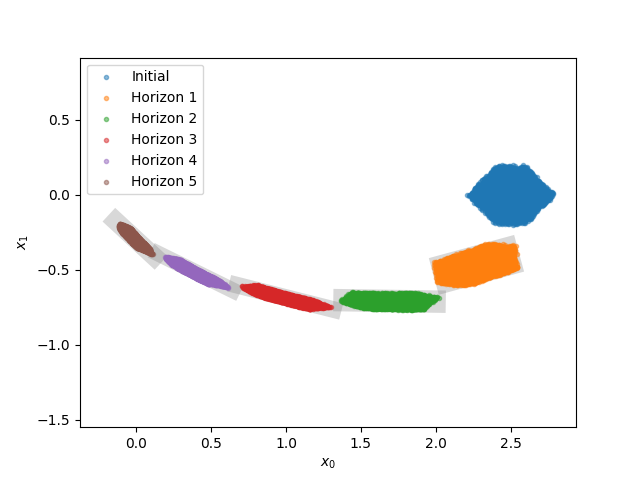}
    \caption{}
    \label{fig:DIcompare}
\end{subfigure}%
\begin{subfigure}{0.24\textwidth}
    \centering
    \includegraphics[width=0.99\textwidth,  trim={1cm 0cm 0.5cm 0cm}]{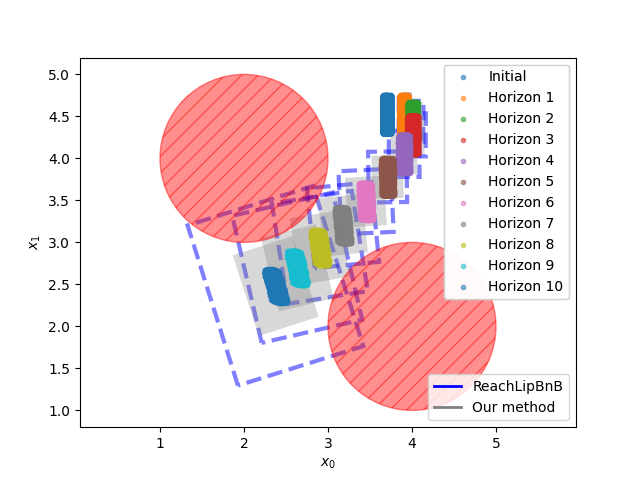}
    \caption{}
    \label{fig:6Qcompare}
\end{subfigure}%
    \caption{\textbf{(a)} Reachability analysis for the Double Integrator on a hexagon input set for five time-steps with $\varepsilon_t = 10^{-3}$. \textbf{(b)} Reachability analysis for the 6D Quadrotor system for 10 time-steps and comparing the results with ReachLipBnB with $\varepsilon_t = 10^{-2}$. 
    }
    \label{fig:reachabilityCompare}
\end{figure}

\paragraph*{Double Integrator (DI).}
The first benchmark is the discrete-time double integrator system from~\cite{hu2020reach}:
\begin{align}
x^{t+1}
=
\begin{bmatrix}
1 & 1\\
0 & 1
\end{bmatrix}x^t
+
\begin{bmatrix}
0.5\\
1
\end{bmatrix}u^t,
\end{align}
where \(x^t\in\mathbb{R}^2\) and \(u^t\in\mathbb{R}\). The controller \(u^t=\pi(x^t)\) is a fully connected neural network with architecture \(2\times 10\times 5\times 5\times 1\) and \(\tanh\) activations, trained to imitate a model predictive control policy.

The initial set is a hexagon centered at \(x_c=[2.5,\ 0]^\top\), represented as
\[
\mathcal{X}
=
\left\{
\begin{bmatrix}
0.1 & 0.1 & 0.1\\
-0.1 & 0 & 0.1
\end{bmatrix}z
+
\begin{bmatrix}
2.5\\
0
\end{bmatrix}
\;\middle|\;
\|z\|_\infty\le 1
\right\}.
\]
Figure~\ref{fig:DIcompare} shows the closed-loop reachable sets over five time steps.

\paragraph*{6D Quadrotor (6Q).}
The second benchmark is the linearized 6D quadrotor system from~\cite{lopez2019arch,hu2020reach}:
\begin{align} x^{t+1}\!=\! &\begin{bmatrix} I_{3\times3} & \Delta t \cdot I_{3\times3}\\ 0_{3\times3} & I_{3\times3} \end{bmatrix} x^{t} + \notag \\ &\Delta t \cdot \begin{bmatrix} & g & 0 & 0 \\ 0_{3\times3} & 0 & -g & 0 \\ & 0 & 0 & 1 \end{bmatrix}^\top\!u^{t}\! +\begin{bmatrix} 0_{5\times1} \\ -g \cdot \Delta t \end{bmatrix}, \notag \end{align}
where \(\Delta t=0.1\) is the sampling time and \(g=9.81\) is the gravitational constant. The task is to stabilize the quadrotor to the origin while avoiding two spherical obstacles. We first design an MPC controller for this task and then train a fully connected neural network policy with architecture \(6\times 32\times 32\times 3\) and \(\tanh\) activations to imitate the MPC controller. The obstacles are unit spheres centered at \([2,\ 4,\ 3]^\top\) and \([4,\ 2,\ 3]^\top\).

Figure~\ref{fig:6Qcompare} shows that our method verifies collision avoidance, whereas ReachLipBnB exceeds its memory limit. This illustrates the benefit of incorporating curvature information into the reachable-set bounds, rather than relying solely on Lipschitz constants.

%
\subsubsection{First-Order vs. Zeroth-Order}

We compare the proposed first-order method with the zeroth-order ReachLipBnB baseline~\cite{entesari2023reachlipbnb} in terms of runtime and number of branches. As shown in Table~\ref{tab:firstVsSecond}, incorporating first-order information substantially reduces both quantities. This improvement is due to the directional sensitivity captured by the gradient term, which yields tighter local over-approximations and enables more aggressive pruning in branch-and-bound. For our method, Lipschitz constants are computed using local \LipMethodName, whereas ReachLipBnB uses LipSDP.

\begin{table}[t]
\caption{Time and branch-count comparisons between the zeroth- and first-order bounding methods for networks with $\tanh$ activation function.}
\label{tab:firstVsSecond}
\setlength{\tabcolsep}{2pt}
\centering
\resizebox{.37\textwidth}{!}
{%
    \begin{tabular}{ c  c  c  c  c } 
    \hline
       & $\varepsilon_t$ & Method & Branches  & Run Time \\ 
     \hline
    \multirow{2}{*}{DI} & \multirow{2}{*}{$10^{-2}$} & Zeroth & 1.9k $\pm$ 0.1k  & 0.95 $\pm$ 0.25 \\ 
            &&First & \textbf{0.6k} $\pm$ 0.06k   & \textbf{0.4} $\pm$ 0.01 \\ 
     \hline
     \multirow{2}{*}{6Q}& \multirow{2}{*}  {$10^{-2}$} & Zeroth & 4.89M $\pm$ 36k  & 678.5 $\pm$ 22.7 \\ 
            &&First & \textbf{0.93M} $\pm$ 16k  & \textbf{385.8} $\pm$ 12 \\ 
     \hline
    \end{tabular}%
}
\end{table}

\subsubsection{Time–Memory Trade-Off}
We now compare the time–memory trade-off in Table~\ref{tab:secondVsSecond}.
LipSDP provides tighter bounds on the Hessian, which improves the relaxation quality at each node and thus reduces the number of branches needed for convergence. However, this comes at the cost of increased runtime due to the computational overhead of solving semidefinite programs.
In contrast, local \LipMethodName incurs negligible computational cost, and while its bounds may be looser, the reduced overhead can lead to better overall runtime compared to LipSDP.

\begin{table}[t]
\caption{Performance comparison between various Lipschitz calculation methods.}
\label{tab:secondVsSecond}
\setlength{\tabcolsep}{2pt}
\centering
\resizebox{.4\textwidth}{!}
{%
    \begin{tabular}{ c  c  c  c  c } 
    \hline
       & $\varepsilon_t$ & Lipschitz & Branches  & Run Time (Lip.) \\ 
     \hline
     \multirow{6}{*}{DI}&\multirow{3}{*}{$10^{-2}$} &
                                    Naive & 2.4k $\pm$ 53  & 0.94 $\pm$ 0.26 (0.008) \\
                                && \LipMethodName & 0.67k $\pm$ 63  & \textbf{0.53} $\pm$ 0.26 (0.02) \\
                                &&LipSDP & \textbf{0.63k} $\pm$ 33  & 1.36 $\pm$ 0.29 (0.8)\\
                            
     \cmidrule{2-5} 
                & \multirow{3}{*}{$10^{-3}$} & Naive & 4.1k $\pm$ 0.1k  & 1.39 $\pm$ 0.27 (0.008) \\
                                && \LipMethodName & \textbf{2.2k} $\pm$ 0.1k  & \textbf{1.03} $\pm$ 0.2 (0.02) \\
                                &&LipSDP & \textbf{2.2k} $\pm$ 0.08k  & 1.83 $\pm$ 0.2 (0.8)\\
     \hline
     
     \multirow{6}{*}{6Q}&  \multirow{3}{*}{$10^{-2}$} 
                        & Naive & 1.67M $\pm$ 64k & 710.6 $\pm$ 30 (0.06)\\
                        && \LipMethodName & 0.93M $\pm$ 16k & 375.1 $\pm$ 5.13 (0.16)\\
                         &&LipSDP & \textbf{0.46M} $\pm$ 24k  & \textbf{181.4} $\pm$ 9.6 (7.4) \\ 

     \cmidrule{2-5} 
     & \multirow{3}{*}{$10^{-3}$}
                        & Naive & 2.3M $\pm$ 42k & 954.2 $\pm$ 20 (0.06)\\
                        && \LipMethodName & 1.7M $\pm$ 96k  & 662.5 $\pm$ 39 (0.17) \\ 
                         &&LipSDP & \textbf{1.1M} $\pm$ 39k  & \textbf{454.5} $\pm$ 16 (7.5) \\ 
     \hline
    \end{tabular}%
}
\end{table}

\subsubsection{Bounding Methods}
\blue{In this experiment, we integrate various incomplete verifiers with our branching scheme on the DI benchmark to isolate and compare the effect of the bounding methods. 
The setup in this experiment is the same as the previous parts, where the $\varepsilon = 10^{-2}$, but only for one time step.
Table~\ref{tab:baselines} summarizes the results, showing that our method consistently achieves faster runtimes than these baselines.}

\begin{table}[t]
\centering
\caption{Run time and number of branches for a single iteration of the double integrator benchmark with $\varepsilon_t = 10^{-2}$. The reported number of branches is averaged over the four subproblems.}
\resizebox{0.45\textwidth}{!}{ 
\begin{tabular}{l c c c}
\hline
Bounding Method  & Run Time (s) & Branches \\
\hline
IBP & 15.6 & 8.5 \\
CROWN & 3.3 & \textbf{1} \\
CROWN-IBP & 3.2 & \textbf{1} \\
\hline
Zeroth & 0.12 & 4 \\
First & \textbf{0.04} & 3.5 \\
\hline
\end{tabular}
}
\label{tab:baselines}
\end{table}

\subsubsection{Branching Heuristics}
    In this experiment, we compare different branching heuristics on the double integrator benchmark.
    Table~\ref{tab:branchHeuristic} shows that the \textit{Best Upper Bound} (\textit{Best UB}) heuristic incurs higher computational cost per branch, which slows down the overall algorithm. However, it can reduce the total number of required branches by more effectively guiding the search.

    \blue{To further broaden the comparison, we evaluate our method against the complete verifier ReachMM~\cite{harapanahalli2023contraction}.
    Figure~\ref{fig:additionalComp} presents the performance of both methods on the double integrator benchmark, each given approximately one second of runtime \blueTwo{(ReachMM terminates at $1.1$s and ours terminates at $0.91$s)}.
    Our method produces tighter reachable sets, especially at later time steps.}

    \blueTwo{
    To ensure fairness, we adjusted the algorithmic parameters of each method so that their runtimes were approximately equal to a budget of $1\mathrm{s}$ ($1.1\mathrm{s}$ for ReachMM and $0.9\mathrm{s}$ for ours). We then compared the total volumes of the resulting bounding boxes under the following setup: for our method, we set $\varepsilon = 0.001$; for ReachMM, we set $cd = 2$ and $id = 0$.
    To further compare the methods' runtime and performance, we conducted additional experiments by performing a grid search over these parameters. The corresponding results are presented in Figure~\ref{fig:ReachMM_chart}.
    }

    \begin{table}[t]
    \caption{Comparison between different branching heuristics.}
    \label{tab:branchHeuristic}
    \setlength{\tabcolsep}{2pt}
    \centering
    \resizebox{.35\textwidth}{!}
    {%
        \begin{tabular}{c  c  c  c  c } 
        \hline
           & $\varepsilon_t$ & Heuristic & Branches  & Run Time \\ 
         \hline
         \multirow{4}{*}{DI}  \quad & \multirow{2}{*}{$10^{-2}$} 
            & Max Length & \textbf{678.8} $\pm$ 63  & \textbf{0.52} $\pm$ 0.26 \\ 
            &&Best UB & 714 $\pm$ 53  & 1.74 $\pm$ 0.22 \\ 
         \cmidrule{2-5}
         & \multirow{2}{*}{$10^{-3}$} 
                &Max Length & 2.25k $\pm$ 0.1k  & \textbf{1} $\pm$ 0.22 \\ 
                &&Best UB & \textbf{2.24k} $\pm$ 0.1k  & 5.01 $\pm$ 0.16 \\ 
         \hline
        \end{tabular}%
    }
    \end{table}

    \begin{figure}[t]
        \centering
        \includegraphics[width=0.6\linewidth]{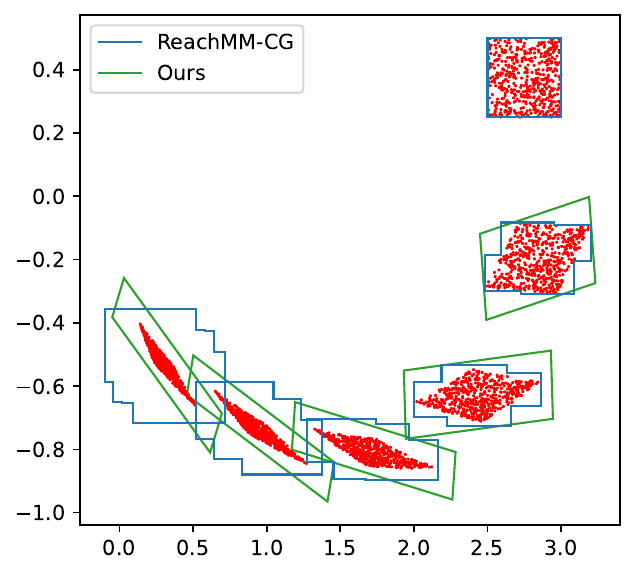}
        \caption{
        \blue{
        Comparison between our first-order method with ReachMM on the double integrator with \texttt{tanh} activations.}}
        \label{fig:additionalComp}
    \end{figure}

    \begin{figure}
        \centering
        \includegraphics[width=0.99\linewidth]{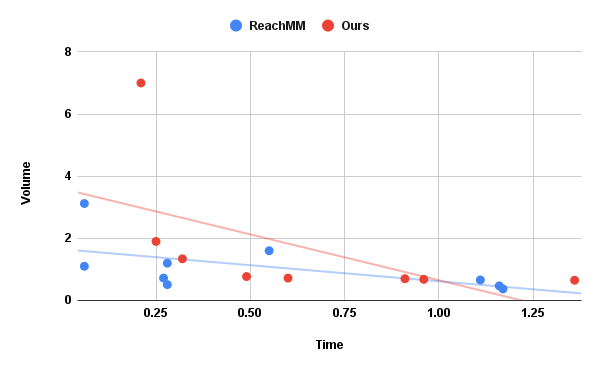}
        \caption{Comparison of the resulting volumes versus the time budget of the methods.}
        \label{fig:ReachMM_chart}
    \end{figure}
\section{Conclusion}
We proposed a novel method to compute provable upper bounds on the second derivative of continuously differentiable neural networks. We used loop transformation to exploit the monotonicity of the activation functions in deriving relatively accurate bounds.  We then derived a derivative-preserving abstraction of the neural network model using Taylor expansion and bounding the Lagrangian remainder. Using this abstraction, we developed a branch-and-bound scheme to perform reachability analysis on neural networks using polyhedrons as reachable set representation. 
We finally conducted numerous experiments to empirically validate the utility of our method. 
\\
In future works, we will explore coupling a state-of-the-art nonlinear system reachability method with ours to achieve even better results. Furthermore, using curvature information of neural networks can also be applied to problems beyond reachability analysis, such as adversarial robustness. 
%

\begin{ack}                               
This work was partially funded by the Johns Hopkins Mathematical Institute for Data Science (MINDS).  
\end{ack}      

\appendix
\section{Appendix}\label{sec:appendix}
\subsection{Proofs}\label{subsec:proofs}
\subsubsection*{Proof of Lemma~\ref{lemma:singlelaterLT}}
Recall from \eqref{eq:liplt2layer} that the loop-transformed network can be written as
\[
f(x) = W^{(2)}\psi(W^{(1)}x;d) \;+\; W^{(2)}DW^{(1)}x,
\]
where \(D=\mathrm{diag}(d)\) and \(\psi(z;d)=\sigma(z)-Dz\) acts componentwise as
\(\psi(z;d) = (\psi_1(z_1;d_1),\dots,\psi_{n_1}(z_{n_1};d_{n_1}))\) with
\(\psi_i(t;d_i)=\sigma_i(t)-d_it\).
The map \(x \mapsto W^{(2)}DW^{(1)}x\) is linear, hence Lipschitz with constant
\(\|W^{(2)}DW^{(1)}\|_p\). 

Since \(\sigma_i \in \mathrm{slope}(\alpha_i,\beta_i)\), \(\psi_i(\cdot;d_i) \in \mathrm{slope}(\alpha_i-d_i,\beta_i-d_i)\). In particular, \(\psi_i\) is
Lipschitz with constant
\[
\ell_i(d_i) := \max\{|\beta_i-d_i|,\;|d_i-\alpha_i|\}.
\]
Let 
\(\ell(d)=(\ell_1(d_1),\dots,\ell_{n_1}(d_{n_1}))\). Using the componentwise Lipschitz property of \(\psi_i\), we can write, for $p \in [1,\infty)$,
\begin{align}
        \|\psi(W^{(1)}x&;d)  \!-\!  \psi(W^{(1)}y;d)\|^p_p
         = \notag \\
        & \sum_{i=1}^{n_1} | \psi_i(W_i^{(1)\top} \! x;d) \! - \! \psi_i(W_i^{(1)\top} \! y;d)|^p \! \leq \notag \\
        & \sum_{i=1}^{n_1} | \ell_i(d_i) W_i^{(1)\top} \! (x - y) |^p = \notag \\
        & \|\mathrm{diag}(\ell(d)) W^{(1)} (x-y)\|^p_p \leq \notag \\
         & \|\mathrm{diag}(\ell(d)) W^{(1)}\|^p_p
        \|x-y\|^p_p.\notag
    \end{align}
Therefore, the map \(x\mapsto W^{(2)}\psi(W^{(1)}x;d)\) is Lipschitz with constant
\(\|W^{(2)}\|_p\|\mathrm{diag}(\ell(d))W^{(1)}\|_p\). Using the triangle inequality and the bounds above, \(f\) is Lipschitz with constant \(L(d)\) given in \eqref{eq: LipLT parameterized bound two layer}.

For $p=\infty$, the same inequality follows directly from
the componentwise Lipschitz property of $\psi_i$. In particular,
\begin{align*}
&\|\psi(W^{(1)}x;d)-\psi(W^{(1)}y;d)\|_\infty\\
&\quad=
\max_{i\in[n_1]}
\left|\psi_i(W_i^{(1)\top}x;d_i)
-\psi_i(W_i^{(1)\top}y;d_i)\right|\\
&\quad\leq
\max_{i\in[n_1]}
\ell_i(d_i)\left|W_i^{(1)\top}(x-y)\right|\\
&\quad=
\|\operatorname{diag}(\ell(d))W^{(1)}(x-y)\|_\infty\\
&\quad\leq
\|\operatorname{diag}(\ell(d))W^{(1)}\|_\infty
\|x-y\|_\infty.
\end{align*}
Thus, the bound holds for all $p\in[1,\infty]$.


\medskip 
\subsubsection*{Proof of Proposition \ref{prop:LT}} \ \\
%
\begin{proof}
    Substituting $D^{(l)} = \mathrm{diag}(\frac{\beta^{(l)}}{2})$ in \eqref{eq:LoopTransLip}, we obtain,
    \begin{align}\label{eq:LocalLTproof1}
        &m^{(l)}(\mathcal{D}^{(l)}) = \|\mathrm{diag}(\frac{\beta^{(l)}}{2}) W^{(l)} \prod_{i=1}^{l-1} \mathrm{diag}(\frac{\beta^{(i)}}{2}) W^{(i)}\| + \notag\\
         &\sum_{j=1}^{l-1} \| \mathrm{diag}(\frac{\beta^{(l)}}{2})W^{(l)} \prod_{i=j+1}^{l-1} \! \mathrm{diag}(\frac{\beta^{(i)}}{2}) W^{(i)}\| \times m^{(j)}(\mathcal{D}^{(j)}).  
    \end{align}
    On the other hand, the local naive method can be expressed by the following recursion.
    \begin{align}
        m^{(l)}(0)&= 
        \|\mathrm{diag}(\beta^{(l)}) W^{(l)}
        \underbrace{\|\prod_{i=1}^{l-1} \|(\mathrm{diag}(\beta^{(i)})) W^{(i)} \|}_{m^{(l-1)}(0)}. \notag 
    \end{align}
    We prove the result in two steps. First, we show by induction that $m^{(l)}(\mathcal{D}^{(l)}) \leq \frac{1}{2}m^{(l)}(0)$ for all hidden layers $l=1,\ldots,L-1$. We then use this bound to establish the desired inequality at the output layer $L$.
    
    We begin with the induction. The base case follows directly from the definitions:
    \[ m^{(1)}(\mathcal{D}^{(1)})
        =\left\|\operatorname{diag}\left(\frac{\beta^{(1)}}{2}\right)
        W^{(1)}\right\|_p
        =\frac{1}{2}m^{(1)}(0).
        \]
    %
    We then assume that $m^{(i)}(\mathcal{D}^{(i)}) \leq \frac{1}{2} m^{(i)}(0)$ holds for all $i \leq l-1,$ and prove that the same inequality for $i=l$. 
    Note that the first term in \eqref{eq:LocalLTproof1} is less than or equal to $\frac{m^{(l)}(0)}{2^{l}}$, by sub-multiplicative property of the matrix norm:
    \begin{align}
        \|\mathrm{diag}(\frac{\beta^{(l)}}{2}) W^{(l)} 
        \prod_{i=1}^{l-1} \mathrm{diag}(\frac{\beta^{(i)}}{2}) W^{(i)}\|
        \leq \notag \\
        \frac{1}{2^l} \prod_{i=1}^{l} \|(\mathrm{diag}(\beta^{(i)})) W^{(i)} \|. \notag
    \end{align}
    Using the same property, one can also show that each one of the terms inside the summation in \eqref{eq:LocalLTproof1} is less than or equal to $\frac{m^{(l)}(0)}{2^{l-j+1}}$:
    \begin{align}
        &\| \mathrm{diag}(\frac{\beta^{(l)}}{2})W^{(l)} \prod_{i=j+1}^{l-1} \! \mathrm{diag}(\frac{\beta^{(i)}}{2}) W^{(i)}\| \times m^{(j)}(\mathcal{D}^{(j)}) = \notag \\
        &\frac{1}{2^{l-j}} \| \mathrm{diag}(\beta^{(l)})W^{(l)} \prod_{i=j+1}^{l-1} \! \mathrm{diag}(\beta^{(i)}) W^{(i)}\| \times m^{(j)}(\mathcal{D}^{(j)})\notag \leq \\
        &\frac{1}{2^{l-j + 1}} \underbrace{\prod_{i=j+1}^{l} \! \| \mathrm{diag}(\beta^{(i)}) W^{(i)}\| \times m^{(j)}(0)}_{m^{(l)}(0)} = \frac{m^{(l)}(0)}{2^{l-j+1}}. \notag
    \end{align}
    where in the second inequality we have used the induction hypothesis $m^{(j)}(\mathcal{D}^{(j)}) \leq \frac{1}{2} m^{(j)}(0)$.
    Therefore, we can write 
    \begin{align}
       m^{(l)}(\mathcal{D}^{(l)}) \leq m^{(l)}(0) ( 2^{- l} + \sum_{j=1}^{l-1} 2^{j-l-1}) 
       = \frac{1}{2} m^{(l)}(0). \notag
    \end{align}
    We have shown that $m^{(l)}(\mathcal{D}^{(l)}) \leq \frac{1}{2} m^{(l)}(0) $ for $l=2, \cdots L-1$.
    Finally, for $l = L$, we have $D'^{(L)} = I$, thus we can write
    \begin{align}\label{eq:LocalLTproof2}
        &m^{(L)}(\mathcal{D}^{(L)}) = \| W^{(L)} \prod_{i=1}^{L-1} \mathrm{diag}(\frac{\beta^{(i)}}{2}) W^{(i)}\| + \notag\\
         &\sum_{j=1}^{L-1} \|W^{(L)} \prod_{i=j+1}^{L-1} \! \mathrm{diag}(\frac{\beta^{(i)}}{2}) W^{(i)}\| \times  
        m^{(j)}(\mathcal{D}^{(j)}). 
    \end{align}
    Hence similar to the proof of the induction, we can show
    \begin{align}
       m^{(L)}(\mathcal{D}^{(L)}) &\leq m^{(L)}(0), \notag
    \end{align}
    which proves the proposition.
\end{proof}

\medskip 
\subsubsection*{Proof of Proposition \ref{prop:lInfBound}} 
    We use induction to prove the proposition.

    \begin{proof}
We prove the result by induction on the number of layers
between $a^{(l)}$ and the network output $z^{(L)}$. For $l=L-1$, the output layer is affine, and hence
$
    \frac{\partial z^{(L)}}{\partial a^{(L-1)}}
    =W^{(L)}.
$
Therefore,
$\left|
    \frac{\partial z^{(L)}}{\partial a^{(L-1)}}
    \right|
    =|W^{(L)}|
    =S^{(L,L-1)}.$
Now suppose that, for some $l\leq L-2$, the
elementwise bound
$
    \left|
    \frac{\partial z^{(L-1)}}{\partial a^{(l)}}
    \right|
    \leq S^{(L-1,l)}
$
holds. By the chain rule,
\begin{align*}
    \frac{\partial z^{(L)}}{\partial a^{(l)}}
    &=
    W^{(L)}
    \mathrm{diag}\left(
        \sigma^{\prime(L-1)}(z^{(L-1)})
    \right)
    \frac{\partial z^{(L-1)}}{\partial a^{(l)}}.
\end{align*}
Taking elementwise absolute values and using the
local slope bounds
$0\leq\sigma_j^{\prime(L-1)}\leq\beta_j^{(L-1)}$, we obtain
\begin{align*}
    \left|
    \frac{\partial z^{(L)}}{\partial a^{(l)}}
    \right|
    &\leq
    |W^{(L)}|
    \mathrm{diag}(\beta^{(L-1)})
    \left|
    \frac{\partial z^{(L-1)}}{\partial a^{(l)}}
    \right|\\
    &\leq
    |W^{(L)}|
    \mathrm{diag}(\beta^{(L-1)})
    S^{(L-1,l)}=S^{(L,l)},
\end{align*}
where the last equality follows from
\eqref{eq:infinitynorm}. Thus, by induction, 
    $|
    \frac{\partial z^{(L)}}{\partial a^{(l)}}
    |
    \leq S^{(L,l)}$ 
for every $l=1,\ldots,L-1$.

Since $z^{(L)}$ is scalar-valued,
$S^{(L,l)}$ is a nonnegative row vector.
Consequently,
\begin{align*}
    \left\|
    \left(
    \frac{\partial z^{(L)}}{\partial a^{(l)}}
    \right)^\top
    \right\|_\infty
    &=
    \max_{j\in[n_l]}
    \left|
    \left(
    \frac{\partial z^{(L)}}{\partial a^{(l)}}
    \right)_j
    \right|\\
    &\leq
    \max_{j\in[n_l]} S_j^{(L,l)}\\
    &=\|S^{(L,l)}\|_1,
\end{align*}
where the last equality uses the induced matrix
$1$-norm of the row vector $S^{(L,l)}$.
This completes the proof.
\end{proof}



\subsubsection*{Proof of Proposition \ref{prop:firstOrderatY}}

\begin{proof}
Let $\lambda=\lambda_{\max}(M)\geq 0$ and
$\delta=x-x_c$. Expanding the objective in
\eqref{eq:subeq2} and dropping terms independent of
$\delta$, the maximization problem is equivalent to
\begin{align}
    \max_{\|\delta\|_p\leq\varepsilon}
    \left\{
    a(y)^\top\delta+\frac{\lambda}{2}\|\delta\|_2^2
    \right\},
    \label{eq:prop5_equiv}
\end{align}
where $a(y)=\nabla J(y)-\lambda(y-x_c)$. For $p=2$, applying the Cauchy--Schwarz inequality gives
\begin{align*}
    a(y)^\top\delta+\frac{\lambda}{2}\|\delta\|_2^2
    &\leq \|a(y)\|_2\|\delta\|_2
    +\frac{\lambda}{2}\|\delta\|_2^2\\
    &\leq \varepsilon\|a(y)\|_2
    +\frac{\lambda}{2}\varepsilon^2,
\end{align*}
where the second inequality follows from
$\|\delta\|_2\leq\varepsilon$ and $\lambda\geq 0$.
For $a(y)\neq 0$, both inequalities are attained by
\(
    \delta^*=\varepsilon\frac{a(y)}{\|a(y)\|_2},
\)
which yields the claimed expression for $x^*(y)$.
If $a(y)=0$, any $\delta^*$ with
$\|\delta^*\|_2=\varepsilon$ is optimal.

For $p=\infty$, the objective in
\eqref{eq:prop5_equiv} is separable across coordinates:
\begin{align*}
    &\max_{\|\delta\|_\infty\leq\varepsilon}
    \left\{
    a(y)^\top\delta\!+\!\frac{\lambda}{2}\|\delta\|_2^2
    \right\}\!=\!
    \sum_{i=1}^{n_x}
    \max_{|\delta_i|\leq\varepsilon}
    \left\{
    a_i(y)\delta_i\!+\!\frac{\lambda}{2}\delta_i^2
    \right\}.
\end{align*}
Since $\lambda\geq 0$, each summand is convex in
$\delta_i$, and hence admits a maximizer at an endpoint
of $[-\varepsilon,\varepsilon]$.
Comparing the objective values at the two endpoints,
their difference is $2\varepsilon a_i(y)$.
Thus, an optimal choice is
\(
    \delta_i^*=\varepsilon\,\mathrm{sign}(a_i(y)),
\)
with the convention $\mathrm{sign}(0)=1$.
Therefore,
\[
    x^*(y)=x_c+\varepsilon\,\mathrm{sign}(a(y)),
\]
which completes the proof.
\end{proof}

\medskip
\subsubsection*{Proof of Proposition \ref{prop:heuristicY}}
    We first state the following lemma.

    \begin{lemma}\label{lemma:optimalY}
Consider the maximizer $x^\star(x_c)$ of
\eqref{eq:subeq2} for $p=\infty$, and let
$\hat{\delta}=x^\star(x_c)-x_c$.
Then $x^\star(x_c)$ remains a maximizer of
\eqref{eq:subeq2} at $y=x_c+\eta\hat{\delta}$
for all
\[
0\leq\eta\leq
\min_i
\frac{|\nabla J(x_c)_i|}
{\lambda_{\max}(M)
(\|\hat{\delta}\|_2+|\hat{\delta}_i|)}.
\]
\end{lemma}
    \begin{proof}
        Choose $y = x_c + \eta\hat{\delta}$.
        To obtain the valid range of $\eta$, we must ensure that $x^*(x_c) = x^*(y)$, where
        \begin{align}
            \begin{cases}
                x^*(x_c)_i = x_{ci} + \varepsilon \mathrm{sign}(\nabla J(x_c)_i) \\
                x^*(y)_i = x_{ci} + \varepsilon \mathrm{sign}(\nabla J(y)_i -  \eta \lambda_{\max}(M) \hat{\delta}_i) \\
            \end{cases}\notag
        \end{align}
        To do so, we use the following inequality
        \begin{align}
            \|\nabla J(y) -  \nabla J(x_c)\|_2 \leq  \lambda_{\max}(M) \eta \|\hat{\delta}\|_2, \notag
        \end{align}
        which implies
        \begin{align}\label{eq:proofIneq}
             -\lambda_{\max}(M) \eta \|\hat{\delta}\|_2 \leq \nabla J(y)_i - \nabla J(x_c)_i \leq  \lambda_{\max}(M) \eta \|\hat{\delta}\|_2.
        \end{align}
        for $i=1,\cdots,n_x$. We use these bounds to ensure that the sign of 
        $\nabla J(y)_i -  \eta \lambda_{\max}(M) \hat{\delta}_i$ is the same as $\nabla J(x_c)$.
        To see this, first assume that $\nabla J(x_c)_i \geq 0$. We then wish to ensure that $\nabla J(y)_i - \eta \lambda_{\max}(M) \hat{\delta}_i) \geq 0$ holds. 
        Using \eqref{eq:proofIneq}
        \begin{align}
            \nabla J(y)_i \! - \! \eta \lambda_{\max}(M) \hat{\delta}_i \!
            \geq \!  
            \nabla J(x_c)_i \! - \! \eta \lambda_{\max}(M) (\|\hat{\delta}\|_2 + \hat{\delta}_i).  \notag
        \end{align}
        The right-hand side is non-negative if 
        $$0 \leq \eta \leq \min_i \frac{\nabla J(x_c)_i}{\lambda_{\max}(M) (\|\hat{\delta}\|_2 + \delta_i)}.$$
        Now assume $\nabla J(x_c)_i \leq 0$. We wish to ensure that $\nabla J(y)_i - \eta \lambda_{\max}(M) \hat{\delta}_i) \leq 0$ also holds. 
        Similarly
        \begin{align}
            \nabla J(y)_i \! - \! \eta \lambda_{\max}(M) \hat{\delta}_i \!
            \leq \!  
            \nabla J(x_c)_i \! + \! \eta \lambda_{\max}(M) (\|\hat{\delta}\|_2 - \hat{\delta}_i). \notag
        \end{align}
        The left-hand side is non-positive if 
        $$0 \leq \eta \leq \min_i \frac{-\nabla J(x_c)_i}{\lambda_{\max}(M) (\|\hat{\delta}\|_2 - \delta_i)}.$$
        Considering both results together, and noting that $\hat{\delta}_i = \varepsilon \mathrm{sign}(\nabla J(x_c)_i)$, the range for $\eta$ can be found
        $$0 \leq \eta \leq \min_i \frac{|\nabla J(x_c)_i|}{\lambda_{\max}(M) (\|\hat{\delta}\|_2 + |\hat{\delta}_i|)},$$
        which concludes the proof.
    \end{proof}

    We utilize this lemma to prove the proposition.
    
    \begin{proof}
    The upper bound obtained at any $y$ can be calculated as
    \begin{align}
        \overline{J}^1_{\mathcal{X}}(y) = 
        J(y) &+ \nabla J(y)^\top (x^*(y) - y) \notag \\
        &+ \frac{\lambda_{\max}(M)}{2} \|x^*(y) - y\|_2^2. \notag
    \end{align}
    Define $\hat{\delta} = x^*(x_c) - x_c$ and set $y = x_c + \eta \hat{\delta}$, where $0 \leq \eta \leq \min_i \frac{|\nabla J(x_c)_i|}{\lambda_{\max}(M) (\|\hat{\delta}\|_2 + |\hat{\delta}_i|)}$. By Lemma \ref{lemma:optimalY} we have $x^*(y) = x^*(x_c)$.
    We can now rewrite the preceding equality as
    \begin{align}\label{eq:opty}
        \overline{J}^1_{\mathcal{X}}(y)\!=\!J(y) \! + \! (1 \! - \! \eta) \nabla J(y)^\top \hat{\delta} 
        \! + \! \frac{\lambda_{\max}(M)}{2} (1 \! - \! \eta)^2 \|\hat{\delta}\|_2^2,
    \end{align}
    where we have used the equality $x^*(y) - y = (1 - \eta)\hat{\delta}$.
    The following quadratic bound can be written for $J(y)$, 
    \begin{align}
        J(y) \leq J(x_c) +  \nabla J(x_c)^\top \eta \hat{\delta} + \eta^2 \frac{\lambda_{\max}(M)}{2}\|\hat{\delta}\|_2^2. \notag
    \end{align}
    By substituting this quadratic bound in \eqref{eq:opty}, we obtain
    \begin{align}
        \overline{J}^1_{\mathcal{X}}(y) \leq 
        J(x_c) +  \nabla J(x_c)^\top \eta \hat{\delta} + \eta^2 \frac{\lambda_{\max}(M)}{2}\|\hat{\delta}\|_2^2 + \notag\\
        (1 - \eta) \nabla J(y)^\top \hat{\delta} + \frac{\lambda_{\max}(M)}{2} (1 - \eta)^2 \|\hat{\delta}\|_2^2. \notag
    \end{align}
    We want to show that this upper bound is less than the original bound, 
    \begin{align}
        \overline{J}^1_{\mathcal{X}}(x_c) = J(x_c) + \nabla J(x_c)^\top \hat{\delta} + \frac{\lambda_{\max}(M)}{2} \|\hat{\delta}\|_2^2. \notag
    \end{align}
    By comparing these two bounds, we obtain the following
    \begin{align}
         (1-\eta) (\nabla J(y) - \nabla J(x_c))^\top \hat{\delta} \leq 2\eta (1-\eta)  \frac{\lambda_{\max}(M)}{2} \|\hat{\delta}\|_2^2. \notag
    \end{align}
    If $\eta \leq 1$, we will have
        $$(\nabla J(y) - \nabla J(x_c))^\top \hat{\delta} \leq 
        \eta \lambda_{\max}(M) \|\hat{\delta}\|_2^2, \notag$$
    which is always true.
    To see this,  it suffices to apply the Cauchy-Schwarz inequality on the left-hand side and use the Lipschitz continuity of $\nabla J$.
    Therefore,
        $\overline{J}^1_{\mathcal{X}}(y) \leq \overline{J}^1_{\mathcal{X}}(x_c),$
    for any 
    $$0 \leq \eta \leq \min \big(\min_i \frac{|\nabla J(x_c)_i|}{\lambda_{\max}(M) (\|\hat{\delta}\|_2 + |\hat{\delta}_i|)} , 1\big).$$
    \vspace{-5mm}
    \end{proof}

\subsection{First-order Upper Bounds for Two-layer Neural Networks}\label{sec:two_layer_net}
    %
    For two-layer neural networks, the matrix $M$ bounding the Hessian is fully known (see $\S$\ref{subsec:Two-layer neural networks}). If $M \preceq 0$, then the maximization problem in \eqref{eq:subeq2} becomes a constrained concave problem that can be solved using methods such as projected gradient ascent.
    In the special case of $p=2$, the problem can be cast as a convex QCQP, which can be solved efficiently. Hence, we focus on the case $M \not\preceq 0$. 
        
        For $p=2$, \eqref{eq:subeq2} is a non-convex QCQP with only a single constraint, implying that strong duality holds \cite{park2017general}. The dual problem can be written as
        \begin{align} \label{eq:dualProblemUB}
            \min_{\lambda} &\quad \lambda \varepsilon^2 + \frac{1}{4}\nabla J(x_c)^\top 
            (\lambda \mathit{I} - \frac{M}{2})^\dagger 
            \nabla J(x_c)\\
            \text{s.t.}
                &\quad \lambda \geq 0, 
                \quad \frac{M}{2} \preceq \lambda \mathit{I},
                \quad \nabla J(x_c) \in \mathcal{R}(\frac{M}{2} - \lambda \mathit{I}). \notag
        \end{align}
        Using the epigraph form and Schur Complements, we can write this problem as an SDP as follows.
        \begin{align}
        \begin{split}
        \min_{\alpha, \lambda} &\quad \alpha\\
        \text{s.t} &\quad
            \begin{bmatrix}
                 \lambda \mathit{I} -\frac{M}{2} & -\frac{1}{2} \nabla J(x_c) \\
                -\frac{1}{2} \nabla J(x_c)^\top & -\lambda \varepsilon^2 + \alpha
            \end{bmatrix} \succeq 0, \quad \lambda \geq 0
        \end{split} \notag
        \end{align}
        Rather than solving the SDP, we can directly solve \eqref{eq:dualProblemUB} using bisection on the dual variable $\lambda \geq 0$.
    
        When $p \in \{1,\infty\}$, the problem becomes a non-convex QP with more than one constraint. The standard approach is to use Lagrangian relaxation and solve the corresponding dual SDP.
        Since the dual function is always convex, a valid upper bound can be calculated at the cost of enduring the duality gap.
        However, the dual SDP will have more than one decision variable, and hence, cannot be efficiently solved using bisection. 
        Another approach, which we adopt in this paper for any $p \neq 2$, is to directly relax the constraint $\|\delta\|_p \leq \varepsilon$ by a single quadratic constraint using \eqref{eq:normIneq}.
        This is equivalent to over-approximating the input set by a minimum volume ellipsoid. Thus, we can write
        \begin{align} 
            \overline{J}^1_{\mathcal{X}}(x_c) &\leq J(x_c) +
            \sup_{\|\delta\|_2 \leq \varepsilon n_0^\kappa} ( \nabla J(x_c)^\top \delta + \frac{1}{2} \delta^\top M \delta), \notag
        \end{align}
        where $\kappa = {\max(0,\frac{1}{2}-\frac{1}{p})}$. We can solve this problem using its dual formulation and bisection, with zero duality gap.

        The case of $M \succeq 0$, is covered in the following proposition.
        
        \smallskip
        \begin{proposition}\label{prop:NDProb}
            For $M \succeq 0$, the solution to the optimization of \eqref{eq:subeq2} lies at some extreme point of the feasible set $\|\delta\|_p \leq \varepsilon$. Specifically, for $p \in \{1, \infty\}$, the solution must lie on one of the vertices. 
        \end{proposition}
        \begin{proof}
            When $M \succeq 0$, the problem becomes a convex maximization problem with a convex constraint set. 
            Bauer's maximum principle \cite{bauer1958minimalstellen} states that a solution to this problem must lie at some extreme point of the set $\|\delta\|_p \leq \varepsilon$. 
            Therefore, for $p \in \{1, \infty\}$ where the feasible set becomes a polyhedron, solving this problem reduces to inspecting all the vertices, which can be efficiently done in low-dimensional spaces.
        \end{proof}

    \begin{figure}[t]
    \centering
    \begin{minipage}[c]{0.55\linewidth}
        \centering
        \includegraphics[width=\linewidth]{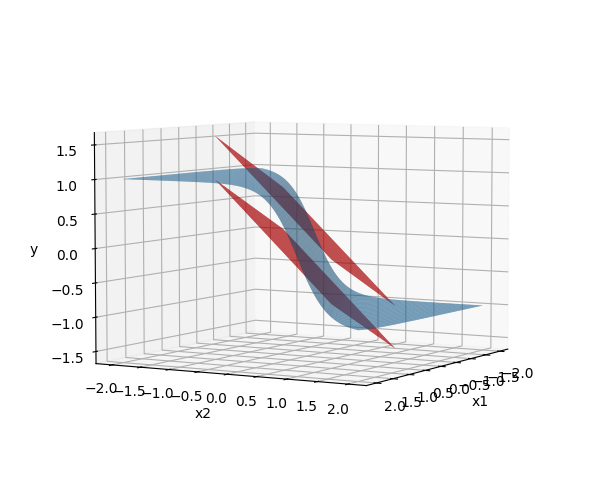}
    \end{minipage}%
    \hspace{0.02\linewidth}
    \begin{minipage}[c]{0.4\linewidth}
        \centering
        \includegraphics[width=\linewidth]{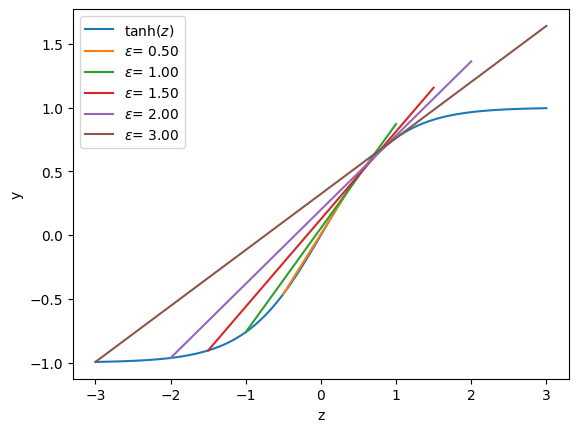}
    \end{minipage}
    \caption{(Left): Visualization of the CROWN upper and lower bounds for the example with $\varepsilon = 1$.
    (Right): The progression of the CROWN upper bound of $\tanh(z)$ for $z \in [-\varepsilon, \varepsilon]$.}
    \label{fig:crown-example}
    \end{figure}

    \blue{
    \subsection{Comparison with Linear Bound Propagation Methods}
    \begin{example} \normalfont
        Consider $J(x) = \sigma(w^\top x)$, where $w = [1, -2]^\top$ and $x \in [-\varepsilon, \varepsilon]^2$. 
        At $x_c=0$, our method matches $J(0)$ and $\nabla J(0)$. 
        In contrast, CROWN constructs an upper bound of the form $U(x) = a_U(\varepsilon) w^\top x + b_U(\varepsilon)$, where {generally,} neither $a_U(\varepsilon)w$ nor $b_U(\varepsilon)$ matches the true gradient or function value {at $x_c$} unless $\varepsilon \to 0$. 
        This distinction is illustrated in Figure~\ref{fig:crown-example}, where our method remains tangent to the true function at the expansion point.
    \end{example}
    }



\bibliographystyle{unsrt}  
\bibliography{bib}   
\end{document}